\documentclass[twoside,11pt]{article}

\usepackage{jmlr2e}

\usepackage{makecell}

\usepackage[utf8]{inputenc}

\usepackage{cite}

\usepackage{nameref}

\usepackage[right]{lineno}

\usepackage{microtype}
\DisableLigatures[f]{encoding = *, family = * }

\usepackage{changepage}

\makeatletter
\renewcommand{\@biblabel}[1]{\quad#1.}
\makeatother

\usepackage{lastpage,fancyhdr,graphicx}
\usepackage{epstopdf}

\usepackage{color}

\definecolor{Gray}{gray}{.25}

\usepackage{graphicx}

\usepackage{sidecap}

\usepackage{wrapfig}
\usepackage[pscoord]{eso-pic}
\usepackage[fulladjust]{marginnote}
\reversemarginpar

\usepackage{graphicx} 
\usepackage{subcaption}
\usepackage{enumerate}

\usepackage{tikz}
\usepackage{xcolor}
\usetikzlibrary{arrows}

\usepackage{mathrsfs}
\usepackage{bbm}

\usepackage{algorithm}
\usepackage[noend]{algpseudocode}
\usepackage{bm,todonotes}

\usepackage{thmtools}
\usepackage{thm-restate}

\usepackage[utf8]{inputenc} 
\usepackage[T1]{fontenc}    
\usepackage{booktabs}       
\usepackage{amsfonts}       
\usepackage{nicefrac}       
\usepackage{microtype}      
\usepackage{amssymb}
\usepackage{graphicx}
\usepackage{color}
\usepackage{amsmath}
\usepackage{pifont}
\usepackage{tabularx}
\usepackage{fancyvrb}
\usepackage{comment}
\usepackage{algorithm}
\usepackage{float}
\usepackage{wrapfig}
\usepackage{mathtools}
\usepackage{paralist} 
\usepackage{multirow}
\usepackage{scrextend}
\usepackage{enumitem}
\usepackage[noorphans]{quoting}
\usepackage{bm}
\usepackage{bbm}
\usepackage{tabularx}
\usepackage{xcolor}
\usepackage{xspace}
\usepackage{mathpazo}
\usepackage{times}
\usepackage{tabu}
\usepackage{icomma}
\usepackage{algpseudocode}
\usepackage{listings}
\usepackage{amssymb}
\usepackage{pifont}
\usepackage{ntheorem}

\DeclareFontFamily{U}{mathx}{\hyphenchar\font45}
\DeclareFontShape{U}{mathx}{m}{n}{
	<5> <6> <7> <8> <9> <10> gen * mathx
	<10.95> mathx10 <12> <14.4> <17.28> <20.74> <24.88> mathx12
}{}
\DeclareSymbolFont{mathx}{U}{mathx}{m}{n}
\DeclareMathSymbol{\intop}  {1}{mathx}{"B3}

\let\temp\phi
\let\phi\varphi
\let\varphi\temp

\newcommand{\R}{\mathbb{R}}

\newcommand{\E}{\mathbb{E}}

\newcommand{\given}{\,|\,}  

\newcommand{\norm}[1]{\Vert#1\Vert}

\newcommand{\ip}[1]{\langle#1\rangle}

\newcommand{\eps}{\varepsilon}

\newcommand{\opt}{\textbf{OPT}}

\DeclarePairedDelimiterX{\infdivx}[2]{(}{)}{%
	#1\;\delimsize\|\;#2%
}

\DeclareMathOperator{\tr}{tr}

\DeclareMathOperator{\Cov}{Cov}

\DeclareMathOperator{\cov}{cov}

\DeclareMathOperator{\Var}{Var}

\DeclareMathOperator{\rank}{rank}

\renewcommand{\algorithmicrequire}{\textbf{Input:}}
\renewcommand{\algorithmicensure}{\textbf{Output:}}

\newcommand{\dist}{\mathcal{D}}

\newcommand{\defeq}{\vcentcolon=}
\newcommand{\eqdef}{=\vcentcolon}

\newcommand{\xxspace}{\mathcal{X}}
\newcommand{\yyspace}{\mathcal{Y}}
\newcommand{\zzspace}{\mathcal{Z}}

\newcommand{\RR}{\mathbb{R}}

\newcommand{\Exp}{\mathbb{E}}

\newcommand{\ind}{\mathbb{I}}
\DeclarePairedDelimiterX{\inp}[2]{\langle}{\rangle}{#1, #2}

\newcommand{\iregion}{\mathcal{R}_{\text{CE}}}
\newcommand{\cregion}{\mathcal{R}_{\text{LS}}}

\newcommand{\deltaya}{\Delta_{Y\mid A}}
\newcommand{\deltaay}{\Delta_{A\mid Y}}
\newcommand{\opty}{E_Y^*}
\newcommand{\opta}{E_A^*}

\renewtheorem{definition}{Definition}[section]
\renewtheorem{conjecture}{Conjecture}[section]
\renewtheorem{example}{Example}[section]
\renewtheorem{theorem}{Theorem}[section]
\renewtheorem{remark}{Remark}[section]
\renewtheorem{lemma}{Lemma}[section]
\renewtheorem{proposition}{Proposition}[section]
\renewtheorem{corollary}{Corollary}[section]

\newtheorem{assumption}{Assumption}[section]

\definecolor{maroon}{rgb}{0.5, 0.0, 0.0}

\newcommand{\error}{\text{err}}

\newcommand{\loss}{\ell}

\newcommand{\risk}{R^{*}}

\makeatletter
\def\thm@space@setup{\thm@preskip=2pt
	\thm@postskip=0pt}
\makeatother

\jmlrheading{23}{2022}{1-\pageref{LastPage}}{9/21; Revised 8/22}{11/22}{21-1078}{Han Zhao, Chen Dan, Bryon Aragam, Tommi S. Jaakkola, Geoffrey J. Gordon, Pradeep Ravikumar}
\ShortHeadings{Fundamental Limits and Tradeoffs in Invariant Representation Learning}{Zhao, Dan, Aragam, Jakkola, Gordon and Ravikumar}

\firstpageno{1}

\begin{document}

\title{Fundamental Limits and Tradeoffs in Invariant Representation Learning}

\author{
       \AND
       \name Han Zhao \thanks{Equal contribution} \email hanzhao@illinois.edu \\
       \addr{University of Illinois Urbana-Champaign}
       \AND
       \name Chen Dan $^*$ \email cdan@cs.cmu.edu \\
       \addr{Carnegie Mellon University}
       \AND
       Bryon Aragam \email bryon@chicagobooth.edu \\
       \addr{University of Chicago}
       \AND
       Tommi S. Jaakkola \email tommi@csail.mit.edu \\
       \addr{Massachusetts Institute of Technology}
       \AND 
       Geoffrey J. Gordon \email ggordon@cs.cmu.edu\\
       \addr{Carnegie Mellon University}
       \AND 
       Pradeep Ravikumar \email pradeepr@cs.cmu.edu \\
       \addr{Carnegie Mellon University}
}

\editor{Adrian Weller}

\maketitle

\begin{abstract}
A wide range of machine learning applications such as privacy-preserving learning, algorithmic fairness, and domain adaptation/generalization among others, involve learning \emph{invariant representations} of the data that aim to achieve two competing goals: (a) maximize information or accuracy with respect to a target response, and (b) maximize invariance or independence with respect to a set of protected features (e.g.\ for fairness, privacy, etc). Despite their wide applicability, theoretical understanding of the optimal tradeoffs --- with respect to accuracy, and invariance --- achievable by invariant representations is still severely lacking. In this paper, we provide an information theoretic analysis of such tradeoffs under both classification and regression settings. More precisely, we provide a geometric characterization of the accuracy and invariance achievable by any representation of the data; we term this feasible region the information plane. We provide an inner bound for this feasible region for the classification case, and an exact characterization for the regression case, which allows us to either bound or exactly characterize the Pareto optimal frontier between accuracy and invariance. Although our contributions are mainly theoretical, a key practical application of our results is in certifying the potential sub-optimality of any given representation learning algorithm for either classification or regression tasks. Our results shed new light on the fundamental interplay between accuracy and invariance, and may be useful in guiding the design of future representation learning algorithms.
\end{abstract}



\section{Introduction}
\label{sec:intro}
A key initial step in any machine learning pipeline is to find a suitable feature representation of the raw data. When we have a specific target response in mind, a common approach is to obtain data representations that maximize predictive accuracy with respect to that target response. In practice, however, one typically faces many competing criteria beyond just accuracy. Examples of such criteria include invariance (typically to a group action), statistical independence (with respect to some feature), privacy (so that certain sensitive information cannot be inferred or otherwise exploited), and out-of-domain generalization (across multiple domains and/or tasks). Along these lines, there has been a surge of recent interest in learning   \emph{invariant representations}~\citep{zemel2013learning,hamm2015preserving,anselmi2016unsupervised,ganin2016domain,johansson2016learning,zhao2019learning} i.e. feature transformations of the input data that aim to balance two goals simultaneously. First, the features should preserve enough information with respect to the target task of interest, e.g., achieve high predictive accuracy. Second, the representations should be invariant to changes in some attribute. 
Often there is a tension between these two competing goals of accuracy and invariance maximization. Despite wide applicability of invariant representations however, a theoretical understanding of the limits and tradeoffs between accuracy and invariance achievable by any representation learning algorithm is severely lacking.

In this paper, towards such an understanding, we provide an inner bound on the  ``feasible region'' of the tuple of accuracy and invariance of any data representation.
The geometric properties of this feasible region can then be related to the Pareto optimal tradeoffs between accuracy and invariance (cf.\ Section~\ref{sec:feasible_region} and Figure~\ref{fig:feasible_region}). Interestingly, given existing representation learning algorithms, our proof technique also implies a constructive approach using randomization to achieve any desired interpolation among the given accuracy and invariance. 
Although our contributions are mainly theoretical, we also demonstrate a practical application of our results in certifying the suboptimality of several existing representation learning algorithms in both classification and regression tasks. Overall, we believe our results take an important step towards a better understanding of the interplay between accuracy and invariance.

\subsection{Our Contributions}
Our main contributions can be summarized as follows:
\begin{itemize}

\item We provide an explicit characterization of the feasible region of accuracy and invariance (which we term the information plane; see Section~\ref{sec:feasible_region} for definitions) for any representation of the data (Sections~\ref{sec:discrete},~\ref{sec:continuous}). This can be used to characterize the tradeoffs between accuracy and invariance, as well as the sub-optimality of any learned representation.


\item While we provide an inner bound of the feasible region for classification, we are able to exactly characterize this feasible region for regression tasks. Moreover, we derive an analytic solution to characterize the Pareto optimal frontier with respect to accuracy and invariance, and discuss sufficient conditions under which this frontier could be achievable. 

\item Finally, we demonstrate a practical application of our theoretical results to certify the sub-optimality of existing representation learning algorithms. In particular, we compare several existing invariant representation learning algorithms on real datasets for both classification and regression tasks using our characterization of the information plane. 
\end{itemize}
We first illustrate our results in the discrete, noiseless setting, which already presents some nontrivial analytical challenges (Section~\ref{sec:discrete}). We then generalize to the noisy, continuous setting, where we can exploit additional machinery in order to refine these results further (Section~\ref{sec:continuous}).

\subsection{Related Work}
There are abundant applications of learning invariant representations in various downstream tasks, including domain adaptation~\citep{ben2007analysis,ben2010theory,ganin2016domain,zhao2018adversarial}, algorithmic fairness~\citep{edwards2015censoring,zemel2013learning,zhang2018mitigating,zhao2019conditional}, privacy-preserving learning~\citep{hamm2015preserving,hamm2017minimax,coavoux2018privacy,xiao2019adversarial}, invariant visual representations~\citep{quiroga2005invariant,mallat2012group,anselmi2016invariance}, causal inference~\citep{johansson2016learning,shalit2017estimating,johansson2020generalization}, and multilingual machine translation~\citep{johnson2017google,aharoni2019massively,zhao2020learning}.

To the best of our knowledge, no previous work studies the particular tradeoff problem in this paper. Closest to our work are results in domain adaptation~\citep{zhao2019learning} and algorithmic fairness~\citep{menon2018cost,zhao2019inherent}, where the authors have shown a lower bound on the classification accuracy on two groups. Compared to these previous results, our work directly characterizes the tradeoff between accuracy and invariance in both classification and regression settings. Furthermore, we also give an approximation to the Pareto frontier between accuracy and invariance in both cases. Another line of related work is the \emph{information bottleneck} method~\citep{tishby2000information,tishby2015deep}, where the goal is to learn representations that preserve sufficient statistics w.r.t.\ the target task of interest while at the same time are compressed as much as possible. The main difference between the information bottleneck and the invariant representations analyzed in this paper is that in the former there is no feature that learned representations are required to be invariant to.


\section{Preliminaries}
\label{sec:setup}
Consider the typical supervised learning setting, where we have an input random vector $X \in \xxspace$, and a target response random variable $Y \in \yyspace$. In either classification or regression, we then aim to estimate a prediction function $f: \xxspace \mapsto \yyspace$ that minimizes $\E\loss(f(X), Y)$ for some loss function  $\loss:\yyspace\times\yyspace\to\R$.
In addition to these input and response variables, in our setting there is a third random variable $A$, with respect to which we desire that our prediction function be invariant. As some examples, $A$ could correspond to potential protected attributes (such as ethnicity or gender of an individual) in algorithmic fairness, or could index the environment or domain in domain adaptation. We let $\dist$ denote the joint distribution  over the triple $(X, A, Y)$, from which our observational data are sampled from. 

Throughout the paper, we will use $H(\cdot)$ to denote the entropy of a random variable and $I(\cdot;\cdot)$ to denote the mutual information between a pair of random variables. For any two random variables $X, Y$ over the same sample space, we also use $X\perp Y$ to denote their statistical independence. As usual, we use $\E[\cdot]$ and $\Var(\cdot)$ to denote the expectation and variance of a random variable, respectively.

Upon receiving the data, the goal of the learner is twofold. On the one hand, the learner aims to accurately predict the target $Y$. On the other hand, it also tries to be insensitive to variation in $A$. To achieve these dual goals, a standard approach in the literature~\citep{zemel2013learning,edwards2015censoring,hamm2015preserving,ganin2016domain,zhao2018adversarial} is \textbf{invariant representation learning}. Formally, we are interested in learning a (possibly randomized) transformation function $Z = g(X)$ that contains as much information as possible about the target $Y$, while at the same time filtering out information related to $A$.

To make these notions more precise, we first define $\risk_{Y}(g)\defeq \inf_{h} \E_{\dist}\loss(h(g(X)), Y)$, as the optimal risk in predicting $Y$ using the feature encoding $Z=g(X)$ under loss $\ell$. We similarly define $\risk_{A}(g)$ to be the optimal risk in predicting $A$ using the feature encoding $Z=g(X)$ under some loss $\ell$. Phrased in terms of these risk functions $\risk_{Y}(g)$ and $\risk_{A}(g)$, the goal of invariant representation learning is to find a representation $g$ such that the response prediction loss $\risk_{Y}(g)$ is minimized while the ``invariance variable'' prediction loss $\risk_{A}(g)$ is maximized.

In this paper, we consider two canonical choices of $\loss$: 
\begin{enumerate}
    \item   (Classification) The cross entropy loss, i.e. $ \loss(y,y')=-y \log(y') - (1-y) \log(1-y')$, typically used  in classification when $Y$ is a discrete variable. Under cross-entropy loss, using standard information-theoretic identities (see Appendix~\ref{sec:details} for more detailed derivations), the risk can be written as
\begin{align}
\begin{aligned}
\risk_{Y}(g) = H(Y\given Z),
\quad
\risk_{A}(g) = H(A\given Z).
\end{aligned}
\label{equ:mi}
\end{align}
    \item   (Regression) The squared loss, i.e. $\loss(y,y')=(y-y')^{2}$, which is suitable for regression tasks with continuous $Y$. Under the squared loss, by the law of total variance (see Appendix~\ref{sec:details} for more detailed derivations), the risk functions can be written as
\begin{align}
\begin{aligned}
\risk_{Y}(g) = \E\Var(Y\given Z),
\quad
\risk_{A}(g) = \E\Var(A\given Z).
\end{aligned}
\label{equ:var}
\end{align}
\end{enumerate}
It is worth pointing out that by using the above risks directly in our framework of analysis, we are focused on the information-theoretic limits of the population errors incurred by the learned representations. Put it in another way, this means that our analysis and results will be oblivious to the optimization methods used to learn the representations, as well as the finite sample effects that exist in practical applications. This is a significant advantage of our approach since our results apply to \emph{any} algorithm used to learn invariant representations, including future developments in this rapidly growing area.

\subsection{Motivating Examples}
\label{sec:examples}
We next discuss several motivating examples to which our framework can be applied. As can be seen from the wide range of these examples, the framework is very generally applicable, and analyzing the tradeoffs discussed in the previous section is of considerable interest.

\begin{example}[Privacy-Preservation]
\label{ex:privacy}
In privacy applications, the goal is to make it difficult to predict sensitive data, represented by the attribute $A$, while retaining information about $Y$. One way to achieve this is to pass information through $Z$, the ``privatized'' data, while simultaneously to preserve the information related to the target application.
\end{example}

\begin{example}[Algorithmic Fairness]
\label{ex:fairness}
In fairness applications, we seek to make predictions about the response $Y$ without discriminating based on the information contained in the protected attributes $A$. For example, $A$ may represent a protected class of individuals defined by, e.g. race or gender. This definition of fairness is also known as \emph{statistical parity} in the literature, and has received increasing attention recently from an information-theoretic perspective~\citep{mcnamara2019costs,zhao2019inherent,dutta2019information}. In particular, one way to achieve this goal is through learning \emph{fair representations}~\citep{zemel2013learning,madras2018learning,zhao2019conditional}. 
\end{example}

\begin{example}[Domain Adaptation / Domain Generalization]
\label{ex:domain}
In domain adaptation and domain generalization, the goal is to train a predictor using labeled data from the source domain that generalizes to the target domain. In this case, $A$ corresponds to the identity (or index) of domains, and the hope here is to learn domain-invariant representations $Z$ that are informative about the target $Y$~\citep{ben2007analysis,ben2010theory,zhao2018adversarial,muandet2013domain,li2018deep}.
\end{example}

\begin{example}[Group Invariance]
In computer vision, it is desirable to learn predictors that are invariant to the action of a group $G$ on the input space. Typical examples include rotation, translation, and scale. By considering random variables $A$ that take their values in $G$, one approach to this problem is to learn representations $Z$ that ``ignore'' changes in $A$ while preserving discriminative features. For example, the recent proposal of contrastive learning~\citep{chen2020simple,he2020momentum,khosla2020supervised} falls into this category.
\end{example}

\begin{example}[Information Bottleneck]
\label{ex:ib}
The \emph{information bottleneck} \citep{tishby2000information} is the problem of finding a representation $Z$ that maximizes the objective $I(Z;Y) - \lambda I(Z;X)$ in an unsupervised manner. This is related to, but not the same as the problem we study, owing to the additional invariant attribute $A$ in our setting. Instead, the goal of information bottleneck methods are to obtain \emph{compressed} representations that also preserve target-related information. Hence at a high-level, information bottleneck seeks to find a tradeoff between \emph{accuracy} and \emph{compression}.
\end{example}




\section{Feasible Region for Accuracy \& Invariance}
\label{sec:feasible_region}

Every representation $Z=g(X)$ can be represented by a point in a two-dimensional \emph{information plane}, which is a notion we introduce in the sequel. The information plane not only facilitates our quantitative analysis, but also enables visualizing the quality of a learned representation with respect to the two competing imperatives of accuracy and invariance, and further, makes it easier to compare different representations.

Let us first consider the classification setting. We have that the risk can be expressed as: 
\[\risk_{Y}(g)=H(Y\given Z) = H(Y) - I(Y;Z).\] 
Noting that $H(Y)$ does not depend on $Z$, we can thus characterize the quality of the representations $Z = g(X)$ with respect to accuracy, i.e.\ the $\risk_{Y}(g)$, by simply measuring the non-constant term $I(Y;Z)$. Similarly, instead of $\risk_A(g)$, we could gauge the quality of the representation $Z = g(X)$ with respect to invariance by simply measuring the non-constant term $I(A;Z)$.

Similarly, for the regression setting, we have the risk defined as:
\begin{equation*}
    \risk_{Y}(g) = \E\Var(Y\given Z) = \Var(Y) - \Var\E[Y\mid Z].
\end{equation*}
As before, since $\Var(Y)$ does not depend on the representations $Z$, we can then characterize the quality of the representations $Z = g(X)$ with respect to accuracy, i.e.\ the $\risk_{Y}(g)$, by measuring the non-constant term $\Var\E[Y\mid Z]$. Likewise, instead of working with $\risk_A(g)$ directly, we could measure the invariance of the representations $Z = g(X)$ by looking at the non-constant term $\Var\E[A\mid Z]$.

We refer to the space of all possible pairs $(I(Y; Z), I(A; Z))$, as we range over possible random vectors $Z$, the information plane. With a slight abuse of terminology, we also refer to the space of all possible pairs $(\Var\Exp[Y\mid Z], \Var\Exp[A\mid Z])$, as we range over possible random vectors $Z$, as the information plane. But we are interested in a particular class of random vectors $Z = g(X)$ that are transformations of the input $X$; these then specify the following regions:
\begin{align*}
\text{(Classification):}&\quad \iregion \defeq\{(I(Y; Z), I(A; Z)) : Z = g(X)\}, \\
\text{(Regression):}&\quad \cregion\defeq\{(\Var\Exp[Y\mid Z], \Var\Exp[A\mid Z]) : Z = g(X)\}.
\end{align*}

Both $\iregion$ and $\cregion$ are subsets of the information plane, which we refer to as \emph{feasible regions} for the classification and regression settings, respectively, since each point in this feasible region corresponds to the quality of any representation $Z = g(X)$ with respect to accuracy, and the quality with respect to invariance. We provide an illustration of the information planes and the feasible regions for both cases in Figure~\ref{fig:feasible_region}.

\begin{figure}[tb]
\centering
\begin{subfigure}{0.45\linewidth}
\centering
    \includegraphics[width=\linewidth]{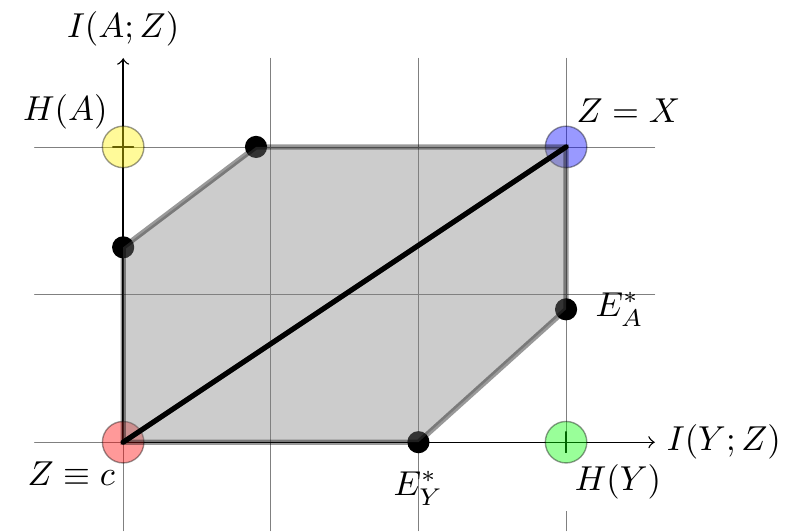}
    \caption{Classification}
\end{subfigure}
~
\begin{subfigure}{0.5\linewidth}
\centering
    \includegraphics[width=\linewidth]{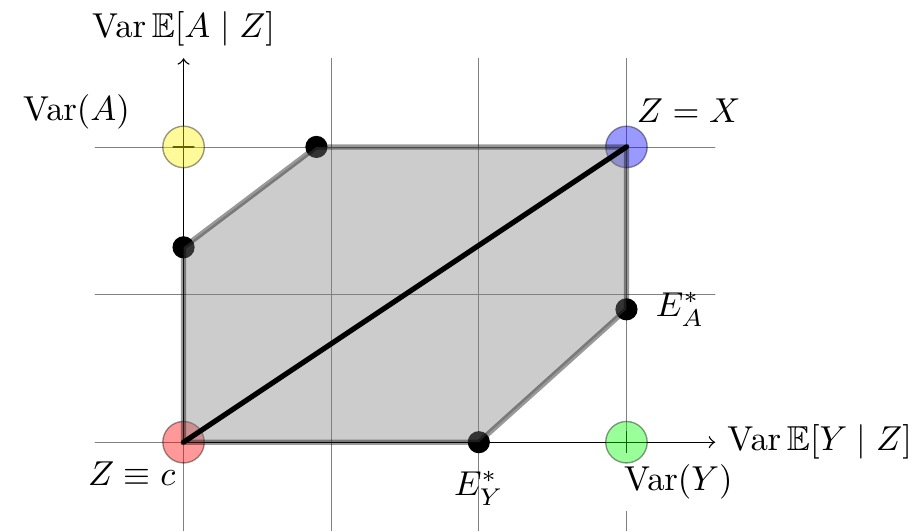}
    \caption{Regression}
\end{subfigure}
\caption{2D information plane. The shaded area correspond to the feasible region.}
\label{fig:feasible_region}
\end{figure}

The primary goal of invariant representation learning is to find representations $Z$ such that $I(Y; Z)$ (resp. $\Var\Exp[Y\mid Z]$) is large and $I(A; Z)$ (resp. $\Var\Exp[A\mid Z]$) is small, corresponding to the lower right corner of Figure~\ref{fig:feasible_region}. By the previous discussion, this is equivalent to finding representations $Z$ that maximize accuracy (i.e. $\Exp_{\dist}[\ell(f(Z), Y)]$) while simultaneously maximizing invariance (i.e. $\Exp_\dist[\ell(f'(Z), A)]$). More precisely, consider the four vertices (not necessarily part of the feasible region) in Figure~\ref{fig:feasible_region}. These four corners have intuitive interpretations:
\begin{itemize}
    \item (Red; Bottom Left) The so-called ``informationless'' regime, in which all of the information regarding both $Y$ and $A$ is destroyed. This is trivially achieved by choosing a constant representation $Z\equiv c$.

    \item (Yellow; Top Left) Here, we retain all of the information in $A$ while removing all of the information about $Y$. This is not a particularly interesting regime for the aforementioned applications.

    \item (Blue; Top Right) The full information regime, where $Z=X$ and no information is lost. This is the ``classical'' setting, wherein information about $A$ is allowed to leak into $Y$.

    \item (Green; Bottom Right) This is the ideal representation that we would like to attain: we preserve all the relevant information about $Y$ while simultaneously removing all the information about $A$.
\end{itemize}

Unfortunately, in general, the ideal representation above (the green corner) may not be attainable due to potential dependency between $Y$ and $A$. That is, there can never exist a representation $Z = g(X)$ that will have invariance and accuracy quality measures corresponding to the green corner. We are thus interested in characterizing how ``close'' we can get to attaining this ideal transformation via any possible representation $Z = g(X)$.
Towards this, it is useful to consider the following extreme points on the boundary of the feasible region:

\newcommand{\optzero}{\opt_0}
\newcommand{\optinf}{\opt_{\infty}}
\begin{itemize}
\item $\opty$:~This point corresponds to a representation $Z$ that maximizes accuracy subject to a hard constraint on the invariance that there be no leakage of information about $A$ into the representation $Z$. In classification, we enforce this via the mutual information constraint $I(A;Z)=0$ and in regression via the conditional variance constraint $\Var\Exp[A\mid Z]=0$. 
\item $\opta$:~This point corresponds to a representation $Z$ that maximizes invariance subject to a hard constraint on the accuracy that there be no loss of information about $Y$ in the representation $Z$. In classification, we enforce this via the mutual information constraint $I(Y;Z)=H(Y)$ and in regression we enforce this via the conditional variance constraint $\Var\Exp[Y\mid Z]=\Var(Y)$. 
\end{itemize}
Now that we have formalized the feasible region, we next geometrically characterize this region, and analytically find the solutions to the extremal points of the region. This characterization will imply a variety of upper and lower bounds on constrained accuracy and invariance, as discussed in Sections~\ref{sec:discrete} and \ref{sec:continuous}. 

\section{Classification}
\label{sec:discrete}
We focus on the case where the original input $X$ contains sufficient information to predict both $Y$ and $A$, so that any loss of accuracy or invariance from a representation $Z = g(X)$ is not due to the non-informative input $X$, but due to the representation itself. This allows us to delineate the effects of the tradeoffs between the two competing goals of accuracy and invariance. Our analysis thus follows the classical setup of Probably Approximately Correct (PAC) learning framework~\citep{valiant1984theory}:
\begin{assumption}
\label{assm:discrete:perfect}
There exist functions $f^*_Y(\cdot)$ and $f^*_A(\cdot)$, such that $Y = f^*_Y(X)$ and $A = f^*_A(X)$. 
\end{assumption}

\vskip0.1in
\noindent The perfect predictors $f^*_Y$ and $f^*_A$ are also known as concepts in the PAC learning framework. Alternatively, this assumption corresponds to the noiseless setting with Bayes optimal classification error being 0.\footnote{This assumption is not necessary; see Appendix~\ref{app:regression}. Essentially, accounting for the Bayes error rate complicates the analysis but does not fundamentally change the conclusions.}
In order to characterize the feasible region $\iregion$, first note that due to the data processing inequality, the following inequalities hold: 
\begin{align*}
    0 \leq I(Y; Z) \leq I(Y; X) &= H(Y), \\
    0 \leq I(A; Z)\leq I(A; X) &= H(A).
\end{align*}
Thus given any $Z$, the point $(I(Y;Z), I(A;Z))$ must lie within a rectangle shown in Figure~\ref{fig:mi0}. The following proposition shows that the feasible region $\iregion$ is convex:
\begin{restatable}{proposition}{iconvex}
The feasible region $\iregion$ for the classification setting is convex.
\label{lemma:convexity}
\end{restatable}

\begin{figure*}[tb]
    \centering
    \begin{subfigure}[b]{0.48\linewidth}
        \includegraphics[width=\linewidth]{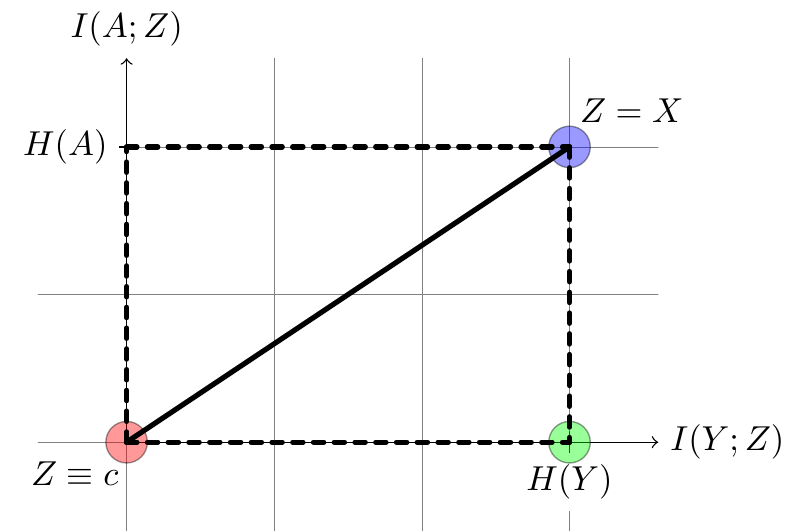}
        \caption{\small Rectangle bounding box.}
        \label{fig:mi0}
    \end{subfigure}
    ~
    \begin{subfigure}[b]{0.48\linewidth}
        \includegraphics[width=\linewidth]{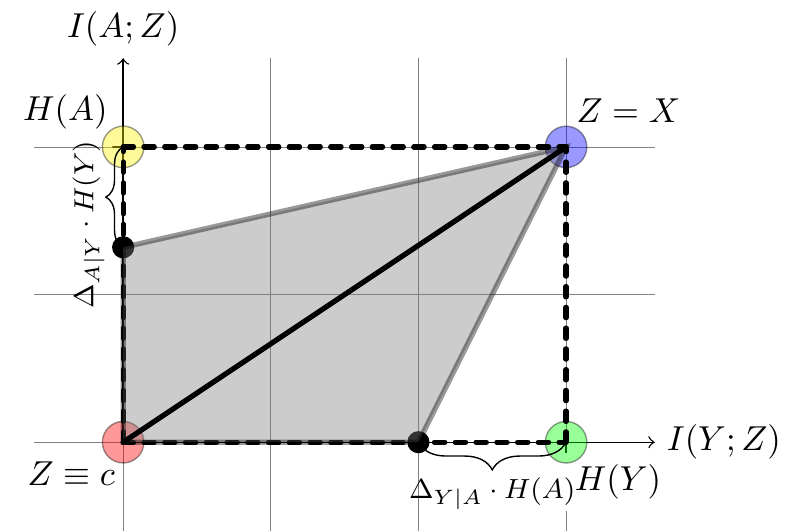}
        \caption{\small Maximal mutual information.}
        \label{fig:mi1}
    \end{subfigure}
    ~
    \begin{subfigure}[b]{0.48\linewidth}
        \includegraphics[width=\linewidth]{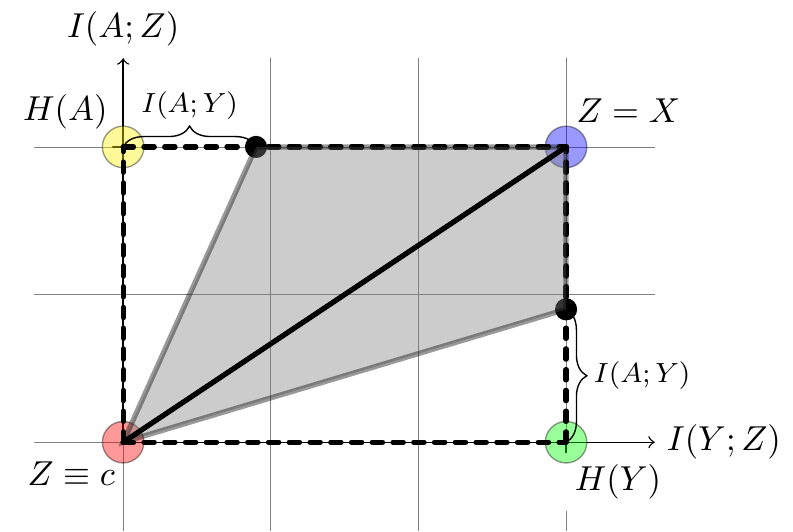}
        \caption{\small Minimum mutual information.}
        \label{fig:mi2}
    \end{subfigure}
    ~
    \begin{subfigure}[b]{0.48\linewidth}
        \includegraphics[width=\linewidth]{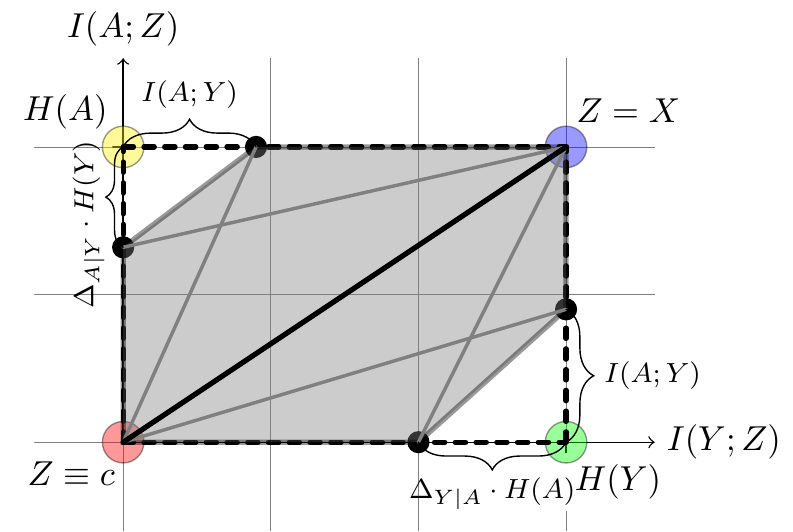}
        \caption{\small The convex polygon characterization of $\iregion$.}
        \label{fig:mi3}
    \end{subfigure}
    \caption{Information plane in classification. Shaded area corresponds to the known feasible region.}
\end{figure*}

\noindent
The convexity of $\iregion$ can be shown by constructing randomized feature transformations to interpolate between any two given representations. Given two feature transformations $Z_0 = g_0(X)$, $Z_1 = g_1(X)$, we can define a uniform random variable $S \sim \text{Bernoulli}(u)$, for some constant $u \in [0,1]$, such that $S$ is independent of $Z_0$ and $Z_1$. Now consider the randomized feature $Z = S Z_0 + (1 - S) Z_1$, i.e.,
\begin{equation}
    Z = \begin{cases}
    Z_0 & \text{If } S \leq u, \\
    Z_1 & \text{otherwise}.
    \end{cases}
\end{equation}
It can then be seen that the invariance and accuracy tuple corresponding to $Z$ is a convex combination of the two points in $\iregion$ corresponding to $Z_0$ and $Z_1$, where the constant $u \in [0,1]$ specifies the convex combination weights. From an algorithmic perspective, this construction also suggests a simple way to achieve a new accuracy-invariance trade-off given existing representations by simply randomizing between them. This geometric nature, convexity, of the feasible region has natural consequences for feasible tradeoffs achievable by some representation. For instance, this entails that all the points on the diagonal of the bounding rectangle are also attainable by some feature representation. 

In the next two sections, we characterize this feasible region in further detail.

\subsection{$\opty$: Maximal Mutual Information under the Independence Constraint}
\label{sec:maxmi}
In this section we explore the first extreme point $\opty$ (Figure~\ref{fig:feasible_region}), which corresponds to maximizing the mutual information of $Z$ w.r.t.\ $Y$ while simultaneously being independent of $A$:
\begin{equation}
\opty := \underset{Z \,:\, I(A;Z) = 0}{\text{maximize}}~I(Y;Z)~.
\label{equ:maxmi}
\end{equation}
First of all, realize that the optimal solution of~\eqref{equ:maxmi} clearly depends on the coupling between $A$ and $Y$. To see this, consider the following two extreme cases:
\begin{example}
If $A = Y$ almost surely, then $I(A; Z) = 0$ directly implies $I(Y;Z) = 0$, hence $\max_Z I(Y;Z) = 0$ under the constraint that $I(A;Z) = 0$.
\label{exp:equal}
\end{example}
\begin{example}
If $A\perp Y$, then $Z = f^*_Y(X) = Y$ satisfies the constraint that $I(A;Z) = I(A;Y) = 0$. Furthermore, $I(Y;Z) = I(Y; Y) = H(Y) \geq I(Y;Z')$, $\forall Z'\neq Y$. Hence $\max_Z I(Y;Z) = H(Y)$.
\label{exp:perp}
\end{example}

\noindent
These two examples show that the optimal solution of~\eqref{equ:maxmi} must involve a quantity that characterizes the dependency between $A$ and $Y$. This motivates the following definition:
\begin{definition}
Define $\deltaya\defeq |\Pr_\dist(Y = 1\mid A = 0) - \Pr_\dist(Y = 1\mid A =1)|$. 
\end{definition}
It is easy to verify the following claims: $0\leq \deltaya\leq 1$ and
\begin{align}
    \deltaya = 0 &\iff A\perp Y, \nonumber\\
    \deltaya = 1 &\iff A = Y \text{ or } A = 1 - Y.
\label{prop:delta}
\end{align}
Armed with this notation, we can now analyze the solution to~\eqref{equ:maxmi}.
\begin{restatable}{theorem}{maxmi}
The optimal value $\opty$ of the program in Eqn. \eqref{equ:maxmi} i.e. the maximal mutual information under the independence constraint is given as: 
\begin{equation}
    \underset{Z \,:\, I(A;Z) = 0}{\max} I(Y;Z) = H(Y) - \deltaya\cdot H(A).
\end{equation}
\label{thm:maxmi}
\end{restatable}

\noindent
As a sanity check of this result, note that if $A\perp Y$, then $\deltaya = 0$, and in this case the optimal value given by Theorem~\ref{thm:maxmi} reduces to $H(Y) - 0\cdot H(A) = H(Y)$, which is consistent with Example~\ref{exp:perp}. Next, consider the other extreme case where $A = Y$. In this case $\deltaya = 1$ and $H(Y) = H(A)$, therefore the optimal solution given by Theorem~\ref{thm:maxmi} becomes $H(Y) - 1\cdot H(A) = H(Y) - H(Y) = 0$. This is consistent with Example~\ref{exp:equal}. 

Furthermore, due to the symmetry between $A$ and $Y$, we can now characterize the locations of the two extremal points on the lower and left boundaries (Figure~\ref{fig:mi1}). In Figure~\ref{fig:mi1}, $\deltaay$ is defined analogously as $\deltaya$ by swapping $Y$ and $A$.

\subsection{$\opta$: Minimal Mutual Information under the Sufficient Statistics Constraint}
\label{sec:minmi}
Next, we characterize $\opta$ in Figure~\ref{fig:feasible_region}. As before, it suffices to solve the following optimization problem, whose optimal solution is $\opta$.
\begin{equation}
\opta := \underset{Z \,:\, I(Y;Z) = H(Y)}{\text{minimize}}~I(A;Z)~.
\label{equ:minmi}
\end{equation}
\begin{restatable}{theorem}{minmi}
The optimal value $\opta$ of the program in Eqn.~\eqref{equ:minmi} i.e. the Minimum Mutual Information under the Sufficient Statistics Constraint is 
\begin{equation}
    \underset{Z \,:\, I(Y;Z) = H(Y)}{\min} I(A;Z) = I(A;Y).
\end{equation}
\label{thm:minmi}
\end{restatable}
Clearly, if $A$ and $Y$ are independent, then the gap $I(A;Y) = 0$, meaning that we can simultaneously preserve all the target related information and filter out all the information related to $A$. We can now characterize the locations of the remaining two extremal points on the top and right boundaries of bounding rectangle (Figure~\ref{fig:mi2}). 

\subsection{Application: Bounds on the Classification Loss}
\label{sec:discrete:bounds}

One application of the optimal solutions derived in Theorems~\ref{thm:maxmi} and~\ref{thm:minmi} is to derive bounds on the risk functions:
\begin{restatable}{corollary}{distbounds}
\label{cor:discrete:bounds}
Let $Z=g(X)$ be any representation. If $I(Z; A) = 0$ then \[\inf_h\Exp[\ell(Y, h(Z))] \ge \deltaya\cdot H(A).\] 
If $I(Y;Z) = H(Y)$, then 
\[\inf_h\Exp[\ell(A, h(Z))]\le H(A\mid Y).\]
\end{restatable}
In the context of learning fair representations where $A$ corresponds to a protected attribute, the first inequality implies that any fair classifier has to incur a population cross-entropy error of at least $\deltaya\cdot H(A)$, even if there exists a perfect classifier $f_Y^*$ over the input $X$.\footnote{If Assumption~\ref{assm:discrete:perfect} is violated, the lower bound will involve the Bayes error rate as well. In either case, $\deltaya\cdot H(A)$ represents an \emph{additional} cost that is incurred by any learning algorithm.} 
 Furthermore, the term $H(A)$ also takes into account the population ratio between different demographic subgroups, in the sense that the more balanced the subgroups in the population, the larger the error lower bound. At first glance this may seem counter-intuitive, but the key is to observe that if the subpopulation proportions are unbalanced, then a classifier could try to achieve invariance by increasing the classification error on the smaller subpopulation. In the limit where one subpopulation has probability zero, then there is no loss in classification error.

In the context of privacy, if we interpret $A$ as a sensitive attribute, and recall that $H(A\mid Y) = \inf_{h'}\Exp[\ell(A, h'(Y))]$ (see Appendix~\ref{sec:details} for more details on this variational form of the conditional entropy), then the second inequality shows that, based on such representation $Z$, the worst case adversary could steal more sensitive information about $A$ by looking at the feature $Z$ instead of the target variable $Y$.

\subsection{Application: Certifying Sub-optimality}
\label{sec:cls:application}
For the aforementioned applications in Section~\ref{sec:examples}, our characterization of the feasible region is critical in order to be able to \emph{certify} that a given model is \emph{not} optimal. For example, given some practically computed representation $Z$ and by using known estimators of the mutual information, it is possible to estimate $I(A;Y)$, $I(Y;Z)$, $I(A;Z)$ in order to directly bound the (sub)-optimality of $Z$ using Figure~\ref{fig:mi3}. In particular, we could compute how far away the point $(I(Y;Z), I(A; Z))$ is from the line segment between $E_Y^*$ and $E_A^*$. This distance lower bounds the distance to the optimal representations on the Pareto frontier. 

More concretely, the following corollary illustrates one example of applying these ideas. We say that $Z$ is suboptimal if there exists a representation $Z'$ such that $I(Y;Z')>I(Y;Z)$ and $I(A;Z')=I(A;Z)$ (resp. $I(A;Z')<I(A;Z)$ and $I(Y;Z')=I(Y;Z)$), i.e. $Z'$ contains strictly \emph{more} information about $Y$ without leaking additional information about $A$ (resp. $Z'$ contains strictly \emph{less} information about $A$ without losing any information about $Y$). In other words, $Z'$ achieves a better accuracy-invariance tradeoff.
\begin{restatable}{corollary}{suboptdist}
\label{cor:subopt:discrete}
Suppose that the representation $Z$ satisfies
\begin{equation}
\label{eq:cor:subopt}
    \frac{I(A;Z)}{I(A;Y)} + \frac{H(Y\given Z)}{\deltaya H(A)}
    > 1.
\end{equation}
Then $Z$ is suboptimal.
\end{restatable}
Geometrically,~\eqref{eq:cor:subopt} states that if $(I(Y;Z), I(A; Z))$ lies in the strict interior of the feasible region in Figure~\ref{fig:mi3}, then the corresponding representation $Z$ is suboptimal. Since each of the quantities in \eqref{eq:cor:subopt} is estimable, this provides a practical certificate of suboptimality for a learned representation $Z$.

In practice, while we do not have access to the population quantities like $I(A, Z)$, we can estimate them using finite samples. The key idea is to use \emph{entropy estimators} in a plug-in fashion. For example, due to the identity $I(A,Z) = H(A)+H(Z) - H(A, Z)$, any consistent entropy estimator $\hat{H}(A), \hat{H}(Z), \hat{H}(A,Z)$ can be turned into a consistent estimator of $ I(A,Z)$. The associated estimators and rates are standard, see the work of \citep{gao2017estimating, ross2014mutual, wu2016minimax}. Besides, $\deltaya$ can also be easily estimated via relative frequencies (i.e. empirical probabilities) for discrete data. Therefore, with any existing entropy estimator, one can construct an estimator of the quantity in Eq.~\eqref{eq:cor:subopt} with finite sample guarantee. For concreteness, here is an example where one can use \citep{wu2016minimax} when $Z$ is discrete with alphabet size at most $k$:
\begin{restatable}{proposition}{certify}
\label{prop:finite_sample}
When $Z$ is discrete with alphabet size at most $k$, with the plug-in estimator~\citep{wu2016minimax} for mutual information applied to the left-hand side of \eqref{eq:cor:subopt} on a sample of size $n$ from the joint distribution, if
\begin{equation}
    \frac{\hat{I}(A;Z)}{\hat{I}(A;Y)} + \frac{\hat{H}(Y\given Z)}{\hat{\Delta}_{Y\mid A} \hat{H}(A)}
    > 1 + \varepsilon,
\end{equation}
then the representation $Z$ is certifiably suboptimal with probability at least $1 - \exp(-\Omega(n\varepsilon^2))$. 
\end{restatable}
Furthermore, if $\deltaya, H(Y), I(A; Y)$ are all lower bounded by some positive constant $c$, the estimator converges at a rate of $1/\sqrt{n}$.
Finite sample versions of other results in the paper follow by similar arguments, as well.

\subsection{Application: Implications for Learning Invariant Representations in Domain Generalization}
One potential application of particular interest of Theorem~\ref{thm:maxmi} is domain generalization. Domain generalization refers to the problem where we aim to train a model on data from a set of source domains so that the learned model can generalize to unseen target
domains. Under this context, the protected attribute $A$ could be understood to be the index of the training domains, where it further specifies from which domain/group the labeled data comes from.

A line of recent works have proposed to learn domain-invariant representations~\citep{albuquerque2019generalizing,deng2020representation,nguyen2021domain,dong2022algorithms} for domain generalization, where intuitively the goal is to learn representations $Z$ from which one cannot infer the domain index $A$. Clearly, in this case if the domain index $A$ also contains discriminative information about the target attribute $Y$, there is a potential loss of accuracy by using domain-invariant representations for domain generalization. More precisely, applying Theorem~\ref{thm:maxmi} to this context, we have that for domain-invariant representations $Z$ $(I(A;Z) = 0)$, the weighted error of any classifier $h$ over $Z$ has the following error lower bound:
\begin{equation}
\label{eq:dg}
    \inf_h\Exp[\ell(Y, h(Z))] = \inf_h~\sum_{a}\Pr(A = a)\cdot\Exp[\ell(Y, h(Z))\mid A = a]\geq \deltaya\cdot H(A).
\end{equation}
Note that although we present our theory for binary classification, the implications here we obtained for domain generalization also hold more generally for multi-class classification problems with more than two source domains. This kind of impossibility result for domain generalization is similar in nature to the one in domain adaptation~\citep[Theorem 4.3]{zhao2019learning}, where the main difference lies in that \citet[Theorem 4.3]{zhao2019learning} considers a sum of balanced errors of the source and target domains whereas the one in~\eqref{eq:dg} is concerned with weighted (given by the size of each source domain) errors among the source domains.


\section{Regression}
\label{sec:continuous}
We now present analogous results in the regression setting, where entropy is replaced by conditional variance. To simplify the presentation, we assume a similar noiseless setting:
\begin{assumption}
\label{assm:continuous:perfect}
	There exist functions $ f_Y^*, f_A^* \in \mathbb{H}$, such that  $Y = f_Y^*(X)$ and $A = f_A^*(X) $, where $\mathbb{H}$ is a given reproducing kernel Hilbert space (RKHS).
\end{assumption}
This assumption is not necessary; the results presented here are special cases of more general results for the noisy setting presented in Appendix~\ref{app:regression}. 

Let $\langle\cdot,\cdot\rangle$ be the canonical inner product in RKHS $\mathbb{H}$. Under this assumption, there exists a feature map $\phi(X)$ and $a\neq 0, y\neq 0$, such that $Y = f_Y^*(X) = \langle \phi(X), y\rangle$ and $A = f_A^*(X) = \langle \phi(X), a \rangle$. This feature map does not have to be finite-dimensional; our analysis applies to infinite-dimensional feature maps as well. 
\begin{figure*}[tb]
    \centering
    \begin{subfigure}[b]{0.46\linewidth}
        \includegraphics[width=\linewidth]{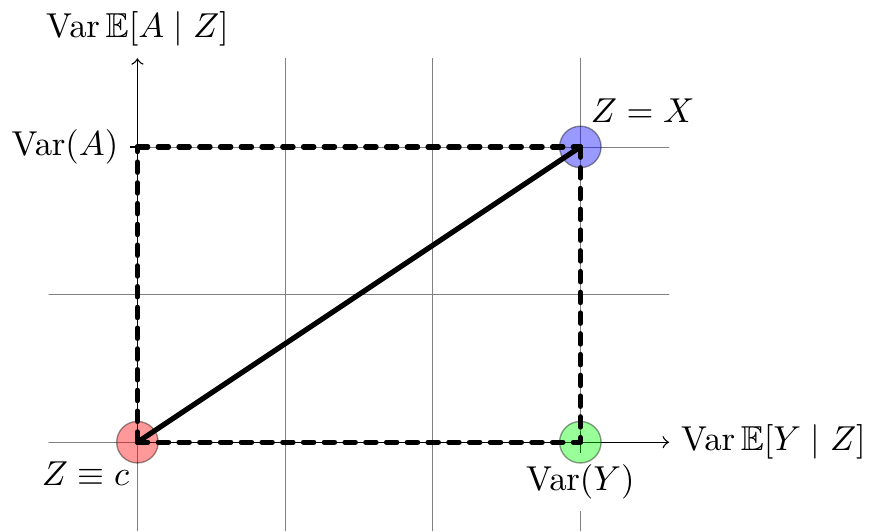}
        \caption{\small Rectangle bounding box.}
        \label{fig:var0}
    \end{subfigure}
    ~
    \begin{subfigure}[b]{0.48\linewidth}
        \includegraphics[width=\linewidth]{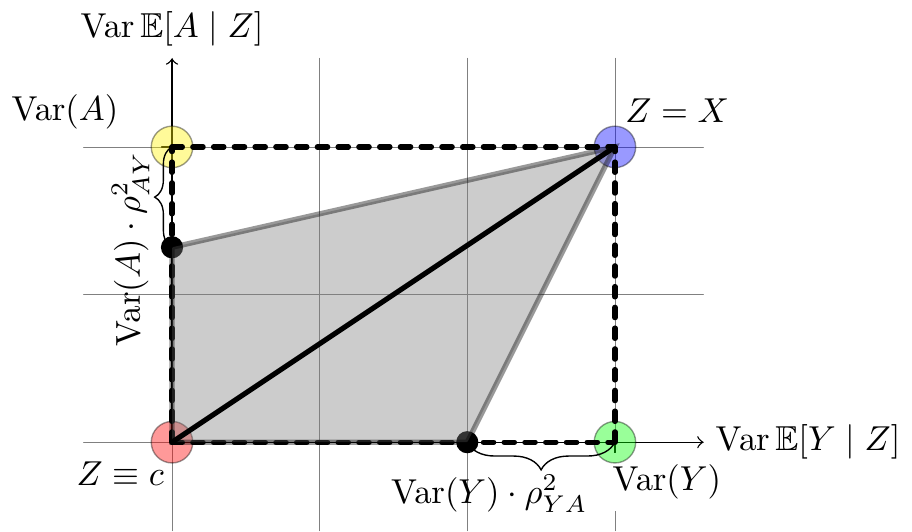}
        \caption{\small Maximal Variance.}
        \label{fig:var1}
    \end{subfigure}
    ~
    \begin{subfigure}[b]{0.46\linewidth}
        \includegraphics[width=\linewidth]{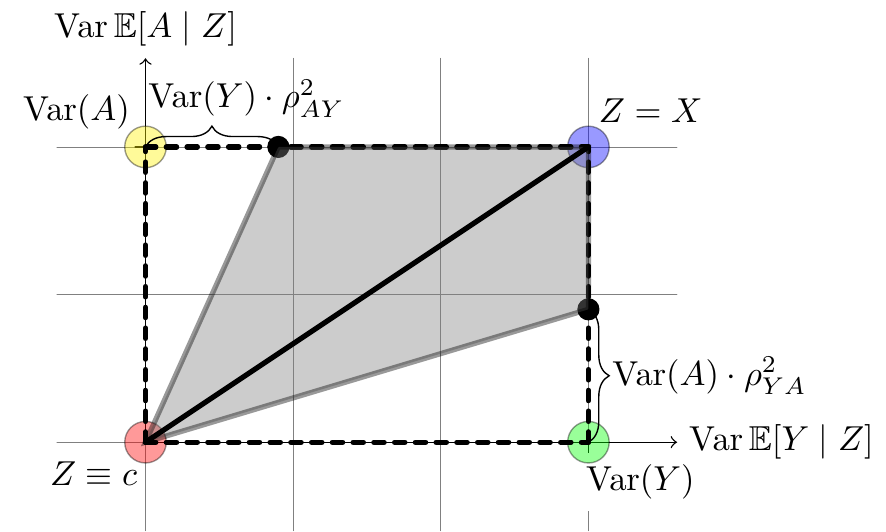}
        \caption{\small Minimum Variance.}
        \label{fig:var2}
    \end{subfigure}
    ~
    \begin{subfigure}[b]{0.48\linewidth}
        \includegraphics[width=\linewidth]{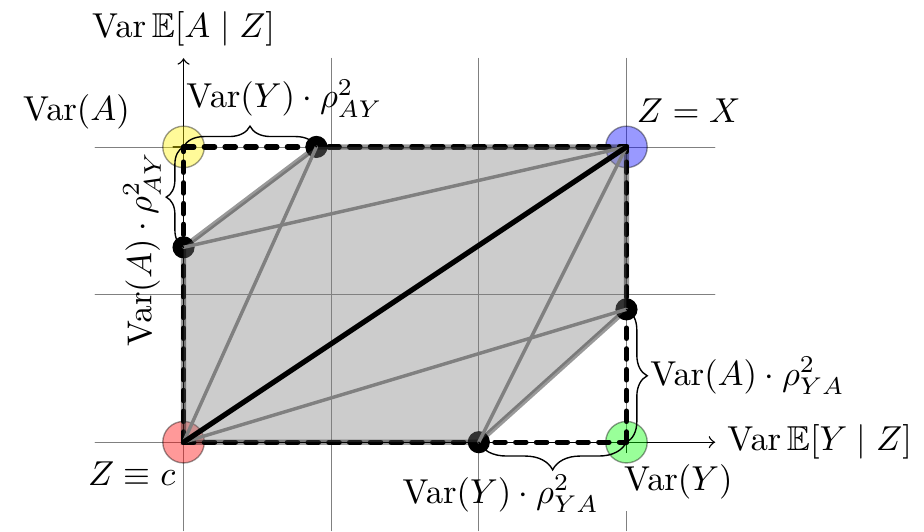}
        \caption{\small The convex polygon characterization of $\cregion$.}
        \label{fig:var3}
    \end{subfigure}
    \caption{Information plane in regression. Shaded area corresponds to the known feasible region.}
\end{figure*}

Similar to our analysis in the classification setting, by the law of total variance, the following inequalities hold:
\begin{align*}
    0 & \leq \Var\Exp[Y\mid Z] \leq \Var\Exp[Y\mid X] = \Var(Y),\\
    0 & \leq \Var\Exp[A\mid Z]\leq \Var\Exp[A\mid X] = \Var(A),
\end{align*}
which means that for any transformation $Z = g(X)$, the point $(\Var\Exp[Y\mid Z], \Var\Exp[A\mid Z])$ must lie within a rectangle shown in Figure~\ref{fig:var0}. To simplify the notation, we define $\Sigma\defeq\Cov(\phi(X), \phi(X))$ to be the covariance operator.

If we consider all the possible feature transformations $Z = g(X)$, then the points $(\Var\Exp[Y\mid Z], \Var\Exp[A\mid Z])$ will form a feasible region $\cregion$. The following lemma shows that the feasible region $\cregion$ is convex:
\begin{restatable}{lemma}{rconvex}
$\cregion$ is convex.
\label{lemma:rconvexity}
\end{restatable}
Again, the convexity of $\cregion$ is guaranteed by a construction of randomized feature transformation. One direct implication of the convexity is that we could interpolate between the performance tuples of existing representation learning methods in the regression setting by allowing randomized representations, as we discussed in the classification setting. Also, both the ``informationless'' origin and the ``full information'' diagonal vertex are attainable. Hence, all the points on the diagonal of the bounding rectangle are also attainable. 

\subsection{$\opty$: Maximal Variance under the Invariance Constraint}
\label{sec:maxvar}
In this section we explore the first extreme point $\opty$ (Figure~\ref{fig:feasible_region}) in the regression setting, which corresponds to maximizing the variance of $Z$ w.r.t.\ $Y$ while simultaneously minimizing that of $A$:
\begin{equation}
\opty \defeq \underset{Z \,:\, \Var\Exp[A\mid Z] = 0}{\text{maximize}}~\Var\Exp[Y\mid Z]~.
\label{equ:maxvar}
\end{equation}
It is clear that the optimal solution of~\eqref{equ:maxvar} depends on the coupling between $A$ and $Y$, and the following theorem precisely characterizes this relationship:
\begin{restatable}{theorem}{maxvar}
The optimal solution of optimization problem~\eqref{equ:maxvar} is upper bounded by 
\begin{align}
    & \underset{Z, \Var\Exp[A\mid Z] = 0}{\max} \Var\Exp[Y\mid Z] \leq \Var(Y) - \frac{\Cov(A, Y)^2 }{\Var(A)} = \Var(Y) (1- \rho_{YA}^2),
    \nonumber
\end{align}
where $\rho_{YA}$ is the correlation coefficient of $Y$ and $A$.
\label{thm:maxvar}
\end{restatable}
Let us do a sanity check of this result: If $A$ and $Y$ are uncorrelated, then $\rho_{YA} = 0$ and hence the gap is 0, and the optimal solution becomes $\Var(Y)$. Next, consider the other extreme case where $A$ is perfectly correlated with $A$: In this case $\rho_{YA} = \pm 1$ so the optimal solution reduces to 0. For readers who are familiar with regression analysis, it can be readily observed that the gap, i.e., $\rho_{YA}^2\Var(Y)$, due to the constraint $\Var\Exp[A\mid Z] = 0$ essentially corresponds to the explained variance of $Y$ by performing an optimal linear regression from $A$, since $\rho^2_{YA}$ equals the R-square of linearly regressing from $A$ to $Y$.

With Lemma~\ref{lemma:rconvexity} and Theorem~\ref{thm:maxvar}, we can now characterize the locations of the two extremal points on the bottom and left boundaries of bounding rectangle in Figure~\ref{fig:var1}. 

\subsection{$\opta$: Minimal Variance under the Sufficient Statistics Constraint}
\label{sec:minmi}
Next, we characterize $\opta$ in the right panel of Figure~\ref{fig:feasible_region}. As before, it suffices to solve the following optimization problem, whose optimal solution is $\opta$.
\begin{equation}
\opta \defeq \underset{Z \,:\, \Var\Exp[Y\mid Z] = \Var(Y)}{\text{minimize}}~\Var\Exp[A\mid Z]~.
\label{equ:minvar}
\end{equation}
\begin{restatable}{theorem}{minvar}
The optimal solution of optimization problem~\eqref{equ:minvar} is lower bounded by 
\begin{equation}
    \underset{Z, \Var\Exp[Y\mid Z] = \Var(Y)}{\min}~\Var\Exp[A\mid Z] \geq \frac{ \langle a, \Sigma y\rangle^2}{\langle y, \Sigma y\rangle} = \Var(A) \cdot \rho_{YA}^2,
\end{equation}
\label{thm:minvar}
where $\rho_{YA}$ is the correlation coefficient of $Y$ and $A$.
\end{restatable}
Again, if $A$ is uncorrelated with $Y$, then the lower bound becomes $0$, meaning that we can simultaneously preserve all the target variance and filter out all the variance related to $A$. This is the no tradeoff regime. On the other hand, if $A$ is perfectly correlated with $Y$, then $\rho_{YA} = \pm 1$, hence the lower bound becomes exactly $\Var(A)$. This extreme case shows that if we would like to preserve all the variation w.r.t.\ the target $Y$, we also have to retain all the information about $A$ as well. This is expected since, again, the lower bound corresponds to the explained variance of $A$ from $Y$ due to the constraint $\Var\Exp[Y\mid Z] = \Var(Y)$.

By symmetry, the updated plot is shown in Figure~\ref{fig:var2}. We plot the full picture of $\cregion$ in Figure~\ref{fig:var3}. Both the constrained accuracy optimal solution and the constrained invariance optimal solution can be readily read from Figure~\ref{fig:var3} as well. Finally, as in the discrete setting, the characterization of the feasible region and the extremal points $\opty$ and $\opta$ have some important consequences, e.g., certifying the suboptimality of a given representation. Here, we note the analogues to Corollary~\ref{cor:discrete:bounds} for the regression setting:
\begin{restatable}{corollary}{contbounds}
\label{cor:continuous:bounds}
Let $Z=g(X)$ be any representation. If $\Var\Exp[A\mid Z] = 0$ then 
\begin{equation}
    \inf_h\Exp[\ell(Y, h(Z))] \geq \Var(Y) \cdot \rho_{YA}^2.
\label{temp:important3}
\end{equation}
Furthermore, if $\Var\Exp[Y\mid Z] = \Var{Y}$, then 
\begin{equation}
   \inf_h\Exp[\ell(A, h(Z))]  \leq \Var(A) (1 - \rho_{YA}^2).
\label{temp:important4}
\end{equation}
\end{restatable}
A similar analogue to Corollary~\ref{cor:subopt:discrete} can be derived in order to certify suboptimality. 

\subsection{An Exact Characterization of the Frontier}
\begin{figure}[tb]
    \centering
    \includegraphics[width=0.8\linewidth]{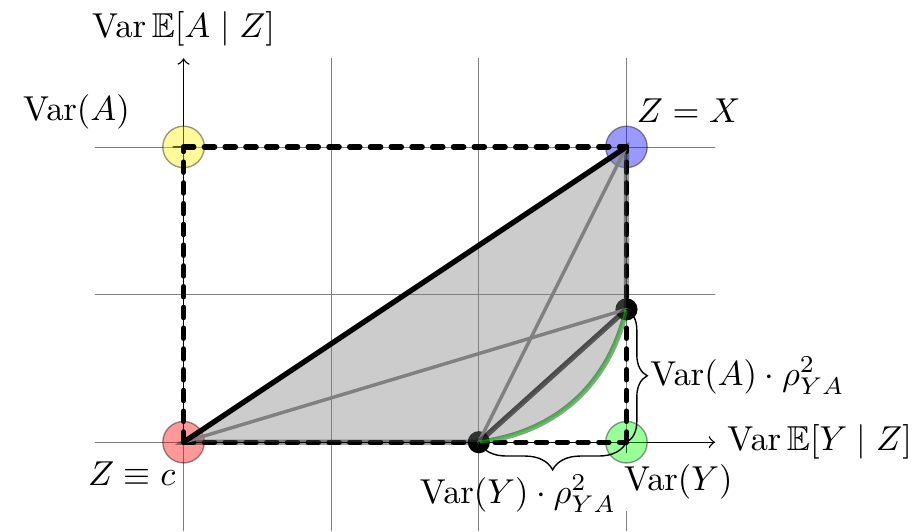}
    \caption{The exact characterization of $\cregion$ in the regression setting. The green convex curve connecting the two black dots is the frontier between $\Var\E[Y|Z]$ and $\Var\E[A|Z]$, given by Theorem~\ref{thm:constrained_upper}.}
    \label{fig:reg_frontier}
\end{figure}
\label{sec:continuous:spectral}
In the regression setting, we can say even more: we can derive a tight upper bound to the constrained problem 
\begin{equation}
\begin{aligned}
& \underset{Z}{\text{maximize}} && \Var\Exp[Y\mid Z] \\
& \text{subject to} && \Var\Exp[A\mid Z] \le c.
\end{aligned}
\label{equ:mmaxvar:relax}
\end{equation}

\begin{restatable}{theorem}{constrained}
	\label{thm:constrained_upper}
	The optimal solution of the constrained problem \eqref{equ:mmaxvar:relax} has the following upper bound: for any $c \in [0, \rho_{YA}^2 \Var(A)]$, we have
	\begin{align}
	\label{equ:constrained_upper}
	\Var\Exp[Y\mid Z] \le \Var(Y)\left(2 \rho_{YA}\sqrt{(1-\rho_{YA}^2)\alpha(1-\alpha)}+1-\alpha - \rho_{YA}^2 + 2\alpha \rho_{YA}^2 \right) 
	\end{align}
	where $\alpha \defeq c/\Var(A)$.
\end{restatable}
When this bound is tight (see Section~\ref{sec:continuous:tight}), this fully characterizes the feasible region $\cregion$. Before we move to the discussion of the tightness of these bounds, we pause to make a few comments on the relationship between Theorem \ref{thm:constrained_upper} and previous results.

A notable special case is when $c = 0$, where the upper bound simplifies to:
\begin{align}
\Var\E[Y| Z] &\leq \Var(Y)(1-\rho_{YA}^2),
\end{align}
which reduces exactly to the bound in Theorem \ref{thm:maxvar}.  

We also note that Theorem \ref{thm:minvar} is also closely related: Theorem \ref{thm:minvar} states that to achieve perfect utility (i.e. $\Var\E[Y| Z] = \Var(Y)$), one must have $ \Var\E[A|Z] \ge \rho_{YA}^2 \Var(A)$. Setting $c = \rho_{YA}^2 \Var(A)$ in Theorem \ref{thm:constrained_upper}, the upper bound exactly simplifies to  $\Var(Y)$, the perfect utility. Therefore, Theorem \ref{thm:constrained_upper} interpolates between the two extreme cases (Theorem \ref{thm:maxvar} and \ref{thm:minvar}). Furthermore, one can verify that the upper bound in Theorem~\ref{thm:constrained_upper} is a concave function of $\alpha$, which means that it is a convex curve on the information plane. In fact, it exactly characterizes the frontier between $\Var\E[Y\mid Z]$ and $\Var\E[A\mid Z]$ in the regression setting. We visualize this result in Figure~\ref{fig:reg_frontier}, where the convex curve connecting the two black dots (the two optimal solutions in the extreme cases) corresponds to the bound in Theorem~\ref{thm:constrained_upper}.

The proof is based on a careful characterization of the Lagrangian form of the constrained problem. Define 
 $\opt(\lambda)$ to be the Lagrangian:
\begin{equation}
\opt(\lambda)\defeq \min_{Z}~  \lambda \Var\Exp[A\mid Z] - \Var\Exp[Y\mid Z]
\label{equ:lag}
\end{equation}

By Lagrange duality, $\opt(\lambda)$ provides an explicit upper bound on the optimal solution to the following problem, which is a relaxed version of the problem \eqref{equ:maxvar}. More specifically, for any $c>0$, 
\begin{equation}
\Var\Exp[Y\mid Z] \le -\sup_{\lambda \geq 0}\left\{\opt(\lambda)-\lambda c \right\}.    
\end{equation}

Therefore, the remaining work is to provide a lower bound of $\opt(\lambda)$, for which we establish the following result:
\begin{restatable}{theorem}{lagrangian}
\label{thm:lower_bound}
The optimal solution of the Lagrangian in regression has the following lower bound:
\begin{align}
\small
\label{equ:lag}
    \opt(\lambda) \geq \frac{1}{2}\left\{ \lambda\Var(A) - \Var(Y) - \sqrt{\Var^2(Y) + \lambda^2\Var^2(A) - 2\lambda\Var(A)\Var(Y)(2\rho_{YA}^2 - 1)}\right\}.
\end{align}
\end{restatable}
A few remarks are in order. The key quantity in the lower bound~\eqref{equ:lag} is the correlation coefficient between $A$ and $Y$: $\rho^2_{YA}$, which effectively measures the dependence between the $Y$ and $A$ under the feature covariance $\Sigma$. To see the connection between Theorem~\ref{thm:lower_bound} and Theorem~\ref{thm:minvar}, Theorem~\ref{thm:maxvar}, note that $\lambda = \infty$ corresponds to the hard constraint that $\Var\Exp[A\given Z]=0$, and $\lambda\to0$ corresponds to the hard constraint that $\Var\Exp[Y\given Z]=\Var(Y)$. Of course, $\lambda\in(0,\infty)$ corresponds to the soft constraint that $\Var\Exp[A\mid Z] \le c$. 

\subsection{When are These Bounds Tight?}
\label{sec:continuous:tight}
The proofs of Theorem~\ref{thm:maxvar}, Theorem~\ref{thm:minvar} and Theorem~\ref{thm:lower_bound} all rely on a semidefinite programming (SDP) relaxation, and construct an explicit optimal solution to this relaxation. At a high level, we re-formulate the objective as a linear functional of $V\defeq \Var\Exp[\phi(X)\mid Z]$, which satisfies the semi-definite constraint $0\preceq V\preceq \Sigma = \Var\phi(X)$. Therefore, the optimal value of the SDP is an upper/lower bound of the corresponding objective. In this section we shall discuss when these SDP relaxations are tight. In particular, we show that under the following regularity condition on $(X, \phi)$, the SDP relaxation is \emph{exact}. 
\begin{definition}
\label{defn:regular}
$(X,\phi)$ is called \emph{regular}, if for any positive semidefinite matrix $M$: $0 \preceq M \preceq \Sigma$, there exists (possibly randomized) $Z = g(X)$, such that 
$\Var\E[ \phi(X)\mid Z] = M$.
\end{definition}
One particularly interesting setting where the regularity condition holds is when $\phi(X)$ follows a Gaussian distribution. 
\begin{restatable}{theorem}{gaussian}
\label{thm:regular}
    $(X, \phi)$ is regular if $\phi(X)$ follows a Gaussian distribution.
\end{restatable}
The proof of Theorem~\ref{thm:regular}, again, depends on an ingenious construction of randomized features. We defer the proof of this theorem to Section~\ref{sec:achievability}. Roughly speaking, we first show that under the Gaussian assumption, all the rank-1 matrices $0 \preceq M\preceq \Sigma$ could be attained via deterministic linear transformations. We then show that the rest of such matrices $M$ could be attained by constructing randomized features, which corresponds to a convex combination of all the rank-1 matrices. 

Under the regularity condition, we can show all the bounds in Theorem~\ref{thm:maxvar}, Theorem~\ref{thm:minvar} and Theorem~\ref{thm:lower_bound} are tight. 
\begin{theorem}
\label{thm:achievable}
When $(X, \phi)$ is regular, the upper bound in Theorem~\ref{thm:maxvar}, the lower bound in Theorem~\ref{thm:minvar} and the spectral lower bound in Theorem~\ref{thm:lower_bound} are all achievable. 
\end{theorem}
Combining Theorem~\ref{thm:regular} and Theorem~\ref{thm:achievable}, we now know that if the distribution over $\phi(X)$ in the RKHS is Gaussian, then all of the bounds in Theorem~\ref{thm:maxvar}, Theorem~\ref{thm:minvar} and Theorem~\ref{thm:lower_bound} become tight. Although here we only show that the regularity condition holds under Gaussian assumption, we conjecture that it could be verified under broader settings. More discussions about the tightness of our bounds and the analysis are presented in Section~\ref{sec:achievability} of the appendix.

\section{Numerical Experiments}
\label{sec:experiments}
In this section, we demonstrate the empirical application of our theoretical results in both classification and regression tasks to \emph{certify the suboptimality} of certain representation learning algorithms. To this end, we conduct experiments on two real-world benchmark datasets, the UCI Adult dataset~\citep{asuncion2007uci} for classification and the Law School dataset~\citep{wightman1998lsac} for regression. In what follows we first provide a brief description about these two datasets and the corresponding tasks. 

\subsection{Datasets}
\paragraph{UCI Adult}
The Adult dataset contains 30,162/15,060 training/test instances for income prediction. Each instance in the dataset describes an adult from the 1994 US Census. Attributes of each individual include gender, education level, age, etc. In this experiment we use gender (binary) as the protected attribute, and we preprocess the dataset to convert all the categorical variables into corresponding one-hot representations. The processed data contains 114 attributes. The target variable (income) is also binary: 1 if $\geq$ 50K/year otherwise 0. For the protected attribute $A$, $A = 0$ means Male otherwise Female. In this dataset, the base rates across groups are different: $\Pr(Y = 1\mid A = 0) = 0.310$ while $\Pr(Y = 1\mid A = 1) = 0.113$. The group ratio is also quite imbalanced, with $\Pr(A = 0) = 0.673$ and $\Pr(A = 1) = 0.327$.

\paragraph{Law School}
The Law School dataset contains 1,823 records for law students who took the bar passage study for Law School Admission.\footnote{We use the edited public version of the dataset which can be downloaded here: \url{https://github.com/algowatchpenn/GerryFair/blob/master/dataset/lawschool.csv}} The features in the dataset include variables such as undergraduate GPA, LSAT score, full-time status, family income, gender, etc. In this experiment, we use gender (treated as a continuous variable that takes value in $[0, 1]$) as the protected attribute and undergraduate GPA (continuous) as the target variable. For both variables, we use the mean-squared error as the loss function. We use 80 percent of the data as our training set and the rest 20 percent as the test set. In the Law School dataset, $\Pr(A = 1) = 0.452$, which is quite balanced. The data distribution for different subgroups in the Law School dataset could be found in Figure~\ref{fig:data-dist-law}. From Figure~\ref{fig:data-dist-law}, we can see that the conditional distributions $\Pr(Y\mid A = a)$ are different across different subgroups $A = a \in\{0, 1\}$.
\begin{figure}[htb]
\centering
    \includegraphics[width=0.6\linewidth]{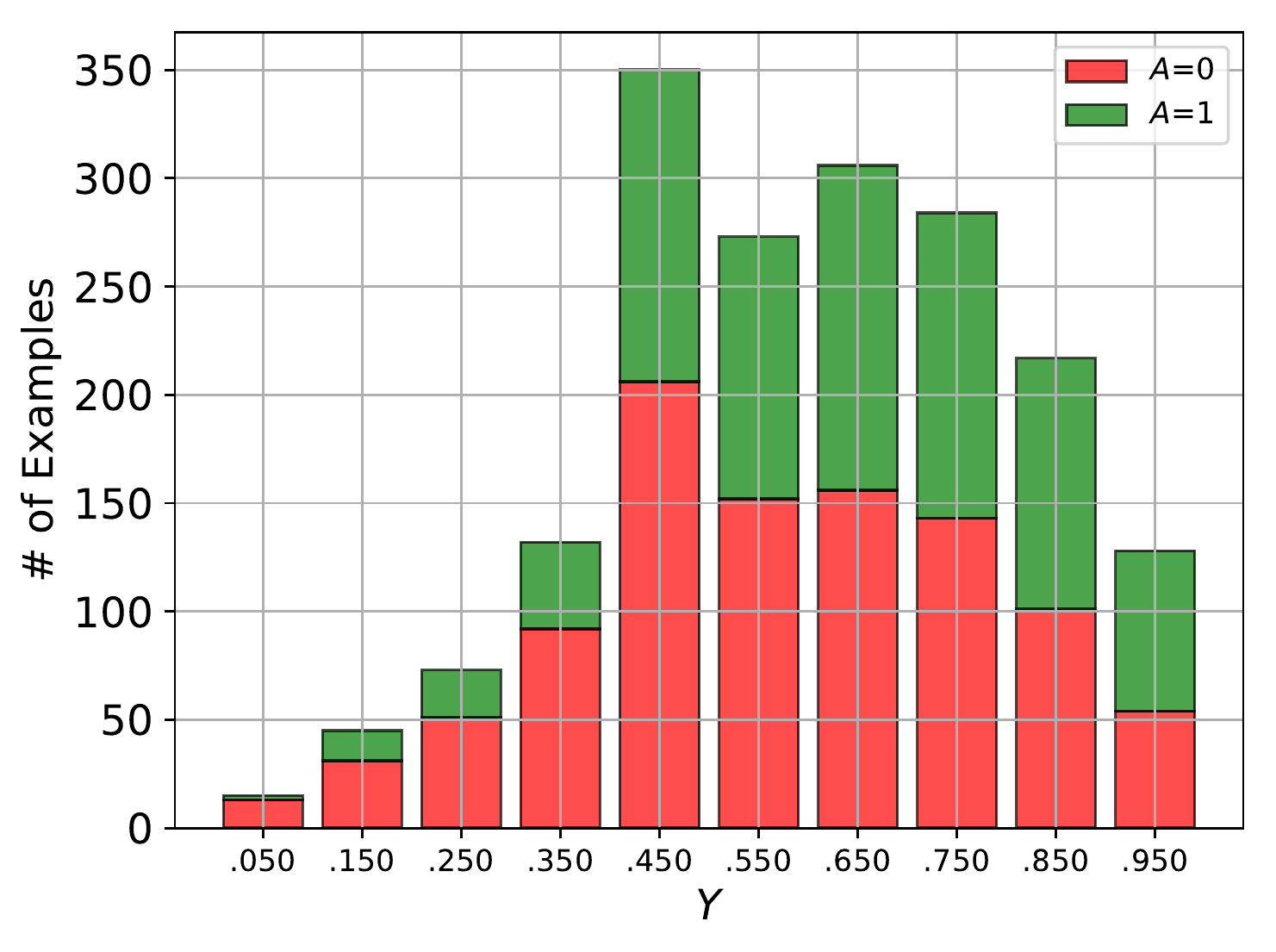}
    \caption{The data distributions of different subgroups in the Law School dataset.}
    \label{fig:data-dist-law}
\end{figure}

\subsection{Representation Learning Algorithms}
To verify that the empirical estimates of the vertices in Figure~\ref{fig:mi3} and Figure~\ref{fig:var3} could be used to certify the suboptimality of certain representation learning algorithms, in this section we apply four different representation learning algorithms to the Adult and Law School datasets and report their corresponding metrics to have a comparison among them. 
\begin{itemize}
    \item   \textbf{Logit Score (Logit)}. The very basic baseline representation learning method we consider is logistic regression. In particular, we first fit a logistic regression model on the Adult dataset and use the logit score of the input data as its representations. Similarly, for the Law School dataset, we consider the \textbf{Ordinary Least Squares (OLS)} for the very basic baseline, where we use the 1-dimensional prediction given by a trained OLS model as the corresponding representations.
    \item   \textbf{Linear Representations (Linear)}. One common representation that has been widely used is the linear representation. In this case, we consider $Z = WX + b$, where $W\in\RR^{d\times p}, b\in\RR^d$ are the model parameters to be learned. 
    \item   \textbf{Multilayer Perceptrons (MLP)}. The representations obtained by MLP are the hidden features of a ReLU network. On the Adult dataset, the network contains one hidden layer. For the Law School dataset, we use a three-hidden-layer network. 
    \item   \textbf{Adversarial Representation Learning (Adv)}. In Adv we use the same network architecture and training hyperparameters as the ones used in MLP, except that now we use an adversarial training method with Wasserstein regularization~\citep{chi2021understanding} to learn invariant representations. 
\end{itemize}
On the Adult dataset, for Logit, LR and MLP, we use the target prediction loss, i.e., cross-entropy in classification and mean-squared error in regression, with stochastic gradient descent to optimize the model parameters. The model parameters (hence the representations) of Adv is obtained via adversarial training, by solving a minimax optimization problem. On the Law School dataset, for OLS, LR and MLP, we use the mean-squared error as the regression loss function and optimize model parameters with stochastic gradient descent. For Adv, again, we obtain the model parameters via adversarial training. 

\subsection{Results and Analysis}
\paragraph{UCI Adult: Classification}
In order to certify the suboptimality and the corresponding distance to the accuracy-invariance frontier of different representation learning algorithms, we need to first estimate the vertices on the information plane as shown in Figure~\ref{fig:mi3}. We do this by using the plug-in estimator. In particular, we compute the conditional probability $\Pr(Y = 1\mid A = a)$ and marginal probabilities $\Pr(Y)$, $\Pr(A)$ from the sample, and then plug them into the formula of $H(Y)$, $H(A)$ and $\deltaya$, respectively. 

We report the results in Table~\ref{tab:adult}. It can be readily verified that the mutual information pair $(I(Y;Z), I(A;Z))$ of all the four methods are consistent with our theoretical upper and lower bounds. Furthermore, none of these algorithms locates on the frontier of the information plane, hence they are all strictly suboptimal. To better understand the strengths of different algorithms, we also visualize the results in Figure~\ref{fig:adult}. From Figure~\ref{fig:adult}, we can see that both MLP and Adv strictly dominate Logit and Linear, which is expected since these two models are richer than the other two baselines. Note that while MLP achieves better $I(Y;Z)$ than Adv, Adv is significantly better in terms of ensuring a small $I(A;Z)$.

\begin{figure}[htb]
    \centering
    \includegraphics[width=0.7\linewidth]{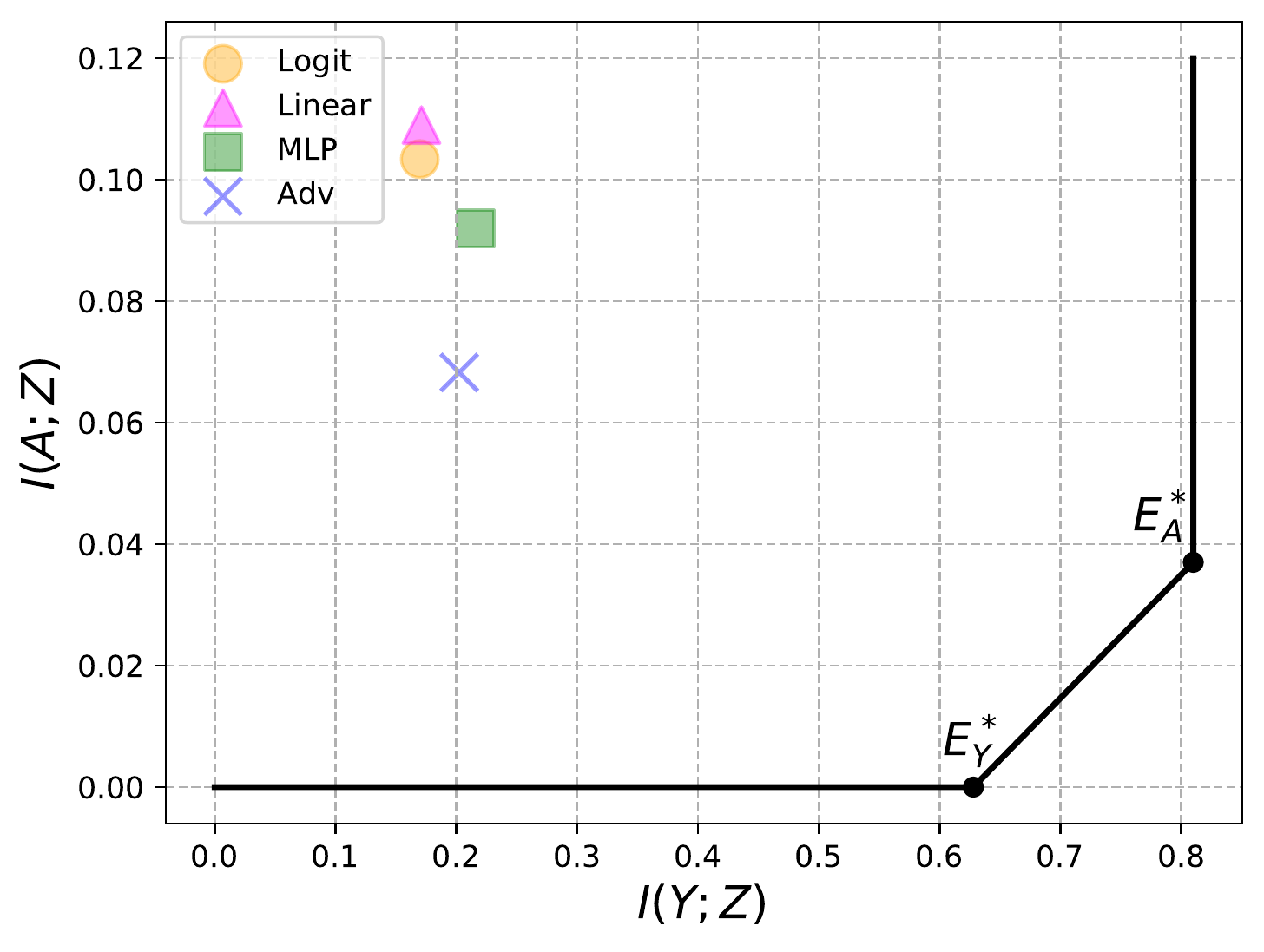}
    \caption{The information plane of Logit, Linear, MLP and Adv on the Adult dataset. On this plane, one method is strictly dominating another if it lies at the bottom-right side of the latter.}
    \label{fig:adult}
\end{figure}
\begin{table}[htb]
\centering
\caption{Numerical results on the \emph{Adult} dataset with Logit, Linear, MLP and Adv. For $I(Y;Z)$, the larger the better. For $I(A;Z)$, the smaller the better.}
\label{tab:adult}
    \begin{tabular}{*7c}\toprule
        & Lower Bound & Logit & Linear & MLP & Adv & Upper Bound \\\midrule
    $I(Y;Z)$ &  0 & 0.170 & 0.171 & 0.216 & 0.202 & $\opty:$ 0.628 \\\midrule 
    $I(A;Z)$ &  $\opta:$ 0.037 & 0.103 & 0.109 & 0.092 & 0.068 & $H(A):$ 0.909\\\bottomrule
    \end{tabular}
\end{table}

\paragraph{Law School: Regression}
Again, similar to the classification task, we first estimate the vertices on the information plane as shown in Figure~\ref{fig:var3}, which includes the empirical variances and the correlation coefficient between $A$ and $Y$. The results for the Law School dataset are shown in Table~\ref{tab:law}. 
As expected, the results are, again, consistent with our theoretical upper and lower bounds on the information plane. But different from the case in the Adult dataset, on the Law School dataset no method is strictly dominating another, and all the four methods are quite far away from the theoretical frontier on the plane. We visualize the results in Figure~\ref{fig:law}. From Figure~\ref{fig:law}, we can conclude that both MLP and Adv are better than OLS and Linear at achieving a large $\Var\E[Y\mid Z]$, but at the cost of increasing $\Var\E[A\mid Z]$ simultaneously. 
\begin{figure}[htb]
    \centering
    \includegraphics[width=0.7\linewidth]{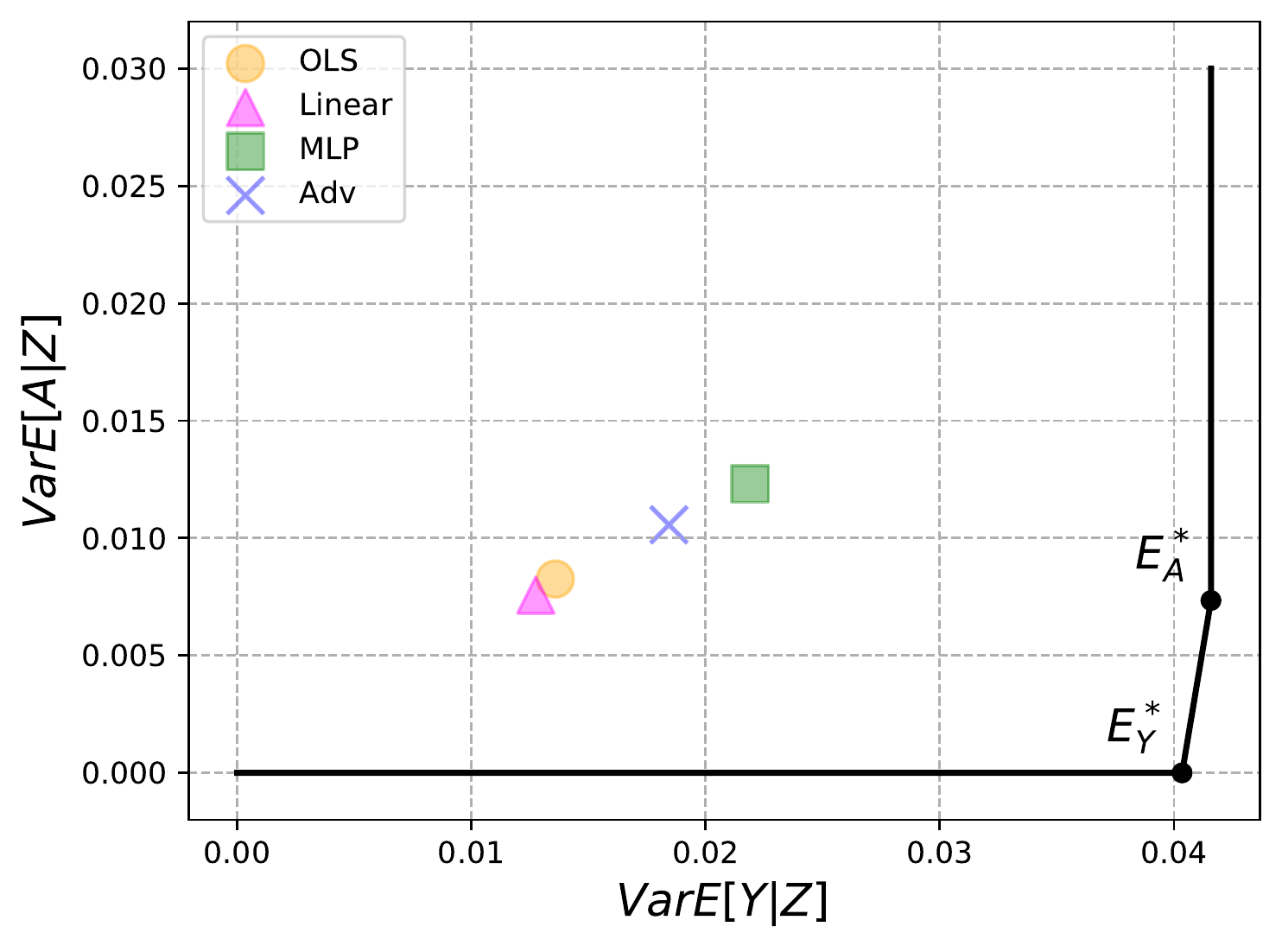}
    \caption{The information plane of OLS, Linear, MLP and Adv on the Law School dataset. On this plane, one method is strictly dominating another if it lies at the bottom-right side of the latter.}
    \label{fig:law}
\end{figure}
\begin{table}[htb]
\centering
\caption{Numerical results on the \emph{Law School} dataset with OLS, Linear, MLP and Adv. For $\Var\E[Y|Z]$, the larger the better. For $\Var\E[A|Z]$, the smaller the better.}
\label{tab:law}
    \begin{tabular}{*7c}\toprule
        & Lower Bound & OLS & Linear & MLP & Adv & Upper Bound \\\midrule
    $\Var\E[Y|Z]$ &  0 & 0.014 & 0.013 & 0.022 & 0.018 & $\opty:$ 0.040 \\\midrule 
    $\Var\E[A|Z]$ &  $\opta:$ 0.007 & 0.008 & 0.008 & 0.012 & 0.010 & $\Var(A):$ 0.248 \\\bottomrule
    \end{tabular}
\end{table}


\section{Conclusion}
In this paper we provide an information plane analysis to study the general and important problem for learning invariant representations in both classification and regression settings. In both cases, we analyze the inherent tradeoffs between accuracy and invariance by providing a geometric characterization of the feasible region on the information plane, in terms of its boundedness, convexity, as well as its its extremal vertices. Furthermore, in the regression setting, we also derive a tight lower bound that for the Lagrangian form of accuracy and invariance, which leads to a complete characterization of the frontier between accuracy and invariance. Given the wide applications of invariant representations in machine learning, we believe our theoretical results could contribute to better understandings of the fundamental tradeoffs between accuracy and invariance under various settings, e.g., domain adaptation, algorithmic fairness, invariant visual representations, and privacy-preservation learning. Furthermore, our analysis on the construction of randomized features also provides a simple way to construct features that can interpolate between existing invariances and accuracies.

So far we mainly focus on the settings where the two attributes of interest are both discrete or continuous. One interesting direction for future work is to consider the mix of these two. For example, could we give similar analytic characterization on the optimal solution of the four constrained problems when one of them is discrete and the other is continuous?

\acks{HZ would like to thank Alexander Kozachinskiy for the discussion in proving Theorem~\ref{thm:maxmi}. HZ and GG would like to acknowledge support from the DARPA XAI project, contract \#FA87501720152 and a Nvidia GPU grant. HZ also thanks the support from a Facebook research award. CD and PR would like to acknowledge the support of NSF via IIS-1955532, IIS-1909816, OAC-1934584, and ONR via N000141812861. CD would like to thank Liu Leqi for many helpful discussions in the early stage of this project.}

\bibliography{reference}

\newpage
\newpage 
\appendix


\section{Proofs of Claims in Section~\ref{sec:feasible_region}}
\label{sec:details}
In this section we give detailed arguments to derive the objective functions of Eq.~\eqref{equ:mi} and~\eqref{equ:var} respectively from the original minimax formulation. First, let us consider the classification setting. 

\paragraph{Classification}
Given a fixed feature map $Z = g(X)$, due to the symmetry between $Y$ and $A$, it suffices for us to consider the case of finding $f$ that minimizes $\Exp_{\dist}[\ell(f\circ g (X), Y)]$, and analogous result follows for the case of finding the optimal $f'$ that minimizes $\Exp_{\dist}[\ell(f'\circ g (X), A)]$ similarly. By definition of the cross-entropy loss, we have:
\begin{align*}
    \Exp_{\dist}[\ell(f\circ g (X), Y)] &= -\Exp_\dist \left[\ind(Y = 0)\log(1-f(g(X))) + \ind(Y=1)\log(g(f(X)))\right] \\
    &= -\Exp_{\dist} \left[\ind(Y = 0)\log(1-f(Z)) + \ind(Y=1)\log(f(Z))\right] \\
    &= -\Exp_Z \Exp_{Y} \left[\ind(Y = 0)\log(1-f(Z)) + \ind(Y=1)\log(f(Z))\mid Z\right] \\
    &= -\Exp_Z  \left[\Pr(Y = 0\mid Z)\log(1-f(Z)) + \Pr(Y=1\mid Z)\log(f(Z))\right] \\
    &= \Exp_Z\left[KL(\Pr(Y\mid Z)~ \|~ f(Z))\right] + H(Y\mid Z) \\
    &\geq H(Y\mid Z).
\end{align*}
It is also clear from the above proof that the minimum value of the cross-entropy loss is achieved when $f(Z)$ is a randomized classifier such that $\Exp[f(Z)] = \Pr(Y=1\mid Z)$. This shows that 
\begin{equation*}
    \min_f\Exp_{\dist}[\ell(f\circ g (X), Y)] = H(Y\mid Z),\quad\text{and}\quad \min_{f'}\Exp_{\dist}[\ell(f'\circ g (X), A)] = H(A\mid Z).
\end{equation*}
To see the second part of Eq.~\eqref{equ:mi}, simply use the identity that $H(Y\mid Z) = H(Y) - I(Y; Z)$ and $H(A\mid Z) = H(A) - I(A; Z)$ with the fact that both $H(Y)$ and $H(A)$ are constants that only depend on the joint distribution $\dist$.

\paragraph{Regression}
Again, given a fixed feature map $Z = g(X)$, because of the symmetry between $Y$ and $A$ let us focus on the analysis of finding $f$ that minimizes $\Exp_{\dist}[\ell(f\circ g (X), Y)]$. In this case since $\ell(\cdot, \cdot)$ is the mean squared error, it follows that 
\begin{align*}
    \Exp_{\dist}[\ell(f\circ g (X), Y)] &= \Exp_\dist \left[(f\circ g(X) - Y)^2\right] \\
    &= \Exp_{\dist} \left[(f(Z) - Y)^2\right] \\
    &= \Exp_Z\left[(f(Z) - \Exp[Y\mid Z])^2\right] + \Exp_Z\left[\Exp_Y[(Y - \Exp[Y\mid Z])^2]\right] \\
    &\geq \Exp_Z\left[\Exp_Y[(Y - \Exp[Y\mid Z])^2]\right] \\
    &= \Exp[\Var(Y\mid Z)],
\end{align*}
where the third equality is due to the Pythagorean theorem. Furthermore, it is clear that the optimal mean-squared error is obtained by the conditional mean $f(Z) = \Exp[Y\mid Z]$. This shows that 
\begin{equation*}
    \min_f\Exp_{\dist}[\ell(f\circ g (X), Y)] = \Exp[\Var(Y\mid Z)],\quad\text{and}\quad \min_{f'}\Exp_{\dist}[\ell(f'\circ g (X), A)] = \Exp[\Var(A\mid Z)].
\end{equation*}
For the second part, use the law of total variance $\Var(Y) = \Exp[\Var(Y\mid Z)] + \Var(\Exp[Y\mid Z])$ and $\Var(A) = \Exp[\Var(A\mid Z)] + \Var(\Exp[A\mid Z])$. Realizing that both $\Var(Y)$ and $\Var(A)$ are constants that only depend on the joint distribution $\dist$, we finish the proof. 

\section{Missing Proofs in Classification (Section~\ref{sec:discrete})}
In what follows we first restate the propositions, lemmas and theorems in the main text, and then provide the corresponding proofs. 

\subsection{Convexity of $\iregion$}
\iconvex*
\begin{proof}
Let $Z_i = g_i(X)$ for $i\in\{0, 1\}$ with corresponding points $(I(Y;Z_i), I(A; Z_i))\in\iregion$. Then we only need to prove that for $\forall u\in[0, 1]$, $(uI(Y;Z_0) + (1-u)I(Y; Z_1), uI(A;Z_0) + (1-u)I(A;Z_1))\in\iregion$ as well.
For any $u\in[0, 1]$, let $S\sim U(0,1)$, the uniform distribution over $(0,1)$, such that $S\perp (Y, A)$. Consider the following randomized transformation $Z = (Z', S)$ where
\begin{equation}
    Z' = \begin{cases}
    Z_0 & \text{If } S \leq u, \\
    Z_1 & \text{otherwise}.
    \end{cases}
\end{equation}
To compute $I(Y; Z)$, we have:
\begin{align*}
    I(Y; Z) =&~ I(Y; Z', S) \\
    =&~ I(S; Y) + I(Z'; Y\mid S) \\
    =&~ I(Z';Y\mid S)  \tag{$S$ is independent of $Y$} \\  
    =&~ \Pr(S\leq u)\cdot I(Y; Z_0) + \Pr(S > u)\cdot I(Y; Z_1) \\
    =&~ u I(Y; Z_0) + (1-u) I(Y;Z_1).
\end{align*}
Similar argument could be used to show that $I(A; Z) = u I(A; Z_0) + (1-u) I(A; Z_1)$. So by construction we now find a randomized transformation $Z = g(X)$ such that $(uI(Y;Z_0) + (1-u)I(Y; Z_1), uI(A;Z_0) + (1-u)I(A;Z_1))\in\iregion$.
\end{proof}

\subsection{Proof of Theorem~\ref{thm:maxmi}}
We proceed to provide the proof that the optimal value of~\eqref{equ:maxmi} is the one given by Theorem~\ref{thm:maxmi}.
\maxmi*
\begin{proof}
For a joint distribution $\dist$ over $(X, A, Y)$ and a function $g:\xxspace\to\zzspace$, in what follows we use $g_\sharp\dist$ to denote the induced distribution of $\dist$ under $g$ over $(Z, A, Y)$. We first make the following claim: without loss of generality, for any joint distribution $g_\sharp\dist$ over $(Z, A, Y)$, we could find $(Z_0, A', Y')\sim g_\sharp\dist$ and a deterministic function $f$, such that $Y' = f(A', Z_0, S)$ where $S\sim U(0, 1)$, $S\perp (A', Z_0)$ and $I(Y'; Z')\geq I(Y; Z)$ with $Z' = (Z_0, S)$. To see this, consider the following construction:
\begin{equation*}
    A', Z_0 \sim \dist(A, Z), \quad S\sim U(0, 1). 
\end{equation*}
Let $(a, z, s)$ be the sample of the above sampling process and construct
\begin{equation*}
    Y' = \begin{cases}
    1 & \text{If $s\leq \Exp[Y\mid A = a, Z = z]$}, \\
    0 & \text{Otherwise}.
    \end{cases}
\end{equation*}
Now it is easy to verify that $(Z_0, A', Y')\sim g_\sharp\dist$ and $\Pr(Y' = 1\mid A' = a, Z_0 = z) = \Exp[Y\mid A = a, Z = z]$. To see the last claim, we have the following inequality hold:
\begin{equation*}
    I(Y'; Z') = I(Y'; Z_0, S) \geq I(Y; Z_0) = I(Y; Z).
\end{equation*}
Now to upper bound $I(Y; Z)$, we have
\begin{equation*}
    I(Y; Z) = H(Y) - H(Y\mid Z),
\end{equation*}
hence it suffices to lower bound $H(Y\mid Z)$. To this end, define
\begin{align*}
    D_0 &\defeq \{z, \eps\in(0, 1)\mid f(0, z, \eps) = 1\}, \\
    D_1 &\defeq \{z, \eps\in(0, 1)\mid f(1, z, \eps) = 1\}.
\end{align*}
Then,
\begin{align*}
    \Pr((z, \eps)\in D_0) &= \Pr(f(0, z, \eps) = 1) \\
                          &= \Pr(f(0, z, \eps) = 1\mid A = 0) \\
                          &= \Pr(f(A, z, \eps) = 1\mid A = 0) \\
                          &= \Pr(Y = 1\mid A = 0).
\end{align*}        
Analogously, the following equation also holds:
\begin{equation*}
    \Pr((z, \eps)\in D_1) = \Pr(Y = 1\mid A = 1).
\end{equation*}
Without loss of generality, assume that $\Pr(Y = 1\mid A = 1)\geq\Pr(Y =1 \mid A = 0)$, then 
\begin{align*}
    \Pr((z, \eps)\in D_1\backslash D_0) &\geq \Pr((z, \eps)\in D_1) - \Pr((z, \eps)\in D_0) \\ 
                                        &= |\Pr(Y =1 \mid A = 1) - \Pr(Y = 1\mid A = 0)|.
\end{align*}
But on the other hand, we know that if $(z, \eps)\in D_1\backslash D_0$, then $f(1, z, \eps) = 1$ and $f(0, z, \eps) = 0$, and this implies that $Y = A$, hence:
\begin{align*}
    H(Y\mid Z) &\geq H(Y\mid Z, S) \\
               &= \Pr((z, \eps)\in D_1\backslash D_0)\cdot H(Y\mid (z, \eps)\in D_1\backslash D_0) + \Pr((z, \eps)\not\in D_1\backslash D_0)\cdot H(Y\mid (z, \eps)\not\in D_1\backslash D_0) \\
               &\geq \Pr((z, \eps)\in D_1\backslash D_0)\cdot H(Y\mid (z, \eps)\in D_1\backslash D_0) \\
               &= \Pr((z, \eps)\in D_1\backslash D_0)\cdot H(A) \\
               &\geq |\Pr(Y =1 \mid A = 1) - \Pr(Y = 1\mid A = 0)|\cdot H(A),
\end{align*}    
which implies that 
\begin{equation*}
    I(Y; Z)\leq H(Y) - |\Pr(Y =1 \mid A = 1) - \Pr(Y = 1\mid A = 0)|\cdot H(A) = H(Y)-\deltaya \cdot H(A).
\end{equation*}
To see that the upper bound could be attained, let us consider the following construction. Denote $\alpha\defeq \Pr(Y = 1\mid A = 0)$ and $\beta \defeq \Pr(Y = 1\mid A = 1)$. Construct a uniformly random $Z\sim U(0, 1)$ and then sample $A$ independently from $Z$ according to the corresponding marginal distribution of $A$ in $\dist$. Next, define:
\begin{equation*}
    Y = \begin{cases}
    1 & \text{if } Z\leq \alpha \wedge A = 0 \text{ or } Z\leq\beta\wedge A = 1, \\
    0 & \text{otherwise}.
    \end{cases}
\end{equation*}
It is easy to see that $Z\perp A$ by construction. Furthermore, by the construction of $Y$, we also have $A, Y\sim\dist(A, Y)$ hold. Since $I(Y; Z) = H(Y) - H(Y\mid Z)$, we only need to verify $H(Y\mid Z) = \deltaya\cdot H(A)$ in this case. Assume without loss of generality $\alpha\leq\beta$, there are three different cases depending on the value of $Z$:
\begin{itemize}
    \item   $Z\leq \alpha$: In this case no matter what the value of $A$, we always have $Y = 1$.
    \item   $Z > \beta$: In this case no matter what the value of $A$, we always have $Y = 0$.
    \item   $\alpha < Z\leq\beta$: In this case $Y = A$, hence the conditional distribution of $Y$ given $Z\in (\alpha, \beta]$ is equal to the conditional distribution of $A$ given $Z\in(\alpha, \beta]$. But by our construction, $A$ is independent of $Z$, which means that in this case the conditional distribution of $A$ given $Z\in (\alpha, \beta]$ is just the distribution of $A$. 
\end{itemize}
Combine all the above three cases, we have:
\begin{align*}
    H(Y\mid Z) &= \Pr(Z\leq \alpha)\cdot H(Y\mid Z \leq \alpha) + \Pr(Z > \beta)\cdot H(Y\mid Z > \beta) + \Pr(\alpha < Z \leq \beta)\cdot H(Y\mid \alpha < Z \leq \beta) \\
               &= 0 + 0 + |\beta - \alpha|\cdot H(A\mid \alpha< Z\leq\beta) \\
               &= |\Pr(Y = 1\mid A = 1) - \Pr(Y =1 \mid A = 0)|\cdot H(A) \\
               &= \deltaya\cdot H(A),
\end{align*}
which completes the proof.
\end{proof}

\subsection{Proof of Theorem~\ref{thm:minmi}}
\minmi*
\begin{proof}
First, realize that $H(Z) \geq I(Y;Z) = H(Y)$ by our constraint. Furthermore, we also know that $0 \leq H(Y\mid Z, A) \leq H(Y\mid Z) = H(Y) - I(Y;Z) = 0$, which means $H(Y\mid Z, A) = 0$. With these two observations, we have:
\begin{align*}
    I(A;Z) =&~ H(Z) - H(Z\mid A) \\
        \geq&~ H(Y) - H(Z\mid A) \\
        \geq&~ H(Y) - H(Y, Z\mid A) \\
        =&~ H(Y) - H(Y\mid A) - H(Y\mid Z, A) \\
        =&~ H(Y) - H(Y\mid A) \\
        =&~ I(A;Y).
\end{align*}
To attain the equality, simply set $Z = f^*_Y(X) = Y$. Specifically, this implies that 1-bit is sufficient to encode all the information for the optimal solution, which completes the proof.
\end{proof}

\subsection{Missing Proofs in Section~\ref{sec:discrete:bounds} and Section~\ref{sec:cls:application}}
In what follows we provide proofs for Corollary~\ref{cor:discrete:bounds} and Corollary~\ref{cor:subopt:discrete}.
\distbounds*
\begin{proof}
For the first inequality, let $\ell(\cdot, \cdot)$ be the cross-entropy loss. If $I(A;Z) = 0$, then by Theorem~\ref{thm:maxmi}
\begin{align*}
    \inf_h~\Exp[\ell(Y, h(Z))] 
    &= H(Y\mid Z) = H(Y) - I(Y;Z) \\
    &= H(Y) - (H(Y) - \deltaya\cdot H(A)) \\
    &= \deltaya\cdot H(A).
\end{align*}

Similarly, if the feature $Z$ preserves all the information w.r.t.\ $Y$, e.g., $I(Y;Z) = H(Y)$, Theorem~\ref{thm:minmi} immediately implies the following bound, 
\begin{align*}
    \inf_h~\Exp[\ell(A, h(Z))] &= H(A\mid Z) = H(A) - I(A; Z) \\
    &= H(A) - I(A; Y) = H(A\mid Y).
\end{align*}
\end{proof}

\suboptdist*
\begin{proof}
From Fig.~\ref{fig:mi3}, the line connecting the points $\opty$ and $\opta$ is given by:
\begin{equation}
I(A;Z) = \frac{I(A;Y)}{\deltaya\cdot H(A)}\big(I(Y;Z) - (H(Y) - \deltaya\cdot H(A))\big).    
\end{equation}
Hence if 
\begin{equation*}
    \frac{I(A;Z)}{I(A;Y)} > 1 - \frac{H(Y\given Z)}{\deltaya\cdot H(A)},
\end{equation*}
then the point $(I(Y;Z), I(A;Z))$ lies strictly on the upper-left side of the line connecting $\opty$ and $\opta$, hence it is suboptimal. 
\end{proof}

In what follows we provide a proof sketch for Proposition~\ref{prop:finite_sample}.
\certify*
\begin{proof}
	It suffices to provide an estimator for 
	\begin{equation}
		\frac{I(A;Z)}{I(A;Y)} + \frac{H(Y\given Z)}{\deltaya\cdot H(A)}
	\end{equation}
	 Since $\deltaya = |\Pr_\dist(Y=1|A=0) -\Pr_\dist(Y=1|A=1) | $ is the difference of two conditional probabilities, for which we can simply estimate with a counting estimator that converges at rate $O(1/\sqrt{n})$.
	
	Using the entropy estimators of \citep{wu2016minimax}, we can estimate all the quantities $H(Y), H(A)$, $H(Z), H(A,Z)$, and $H(A,Y), H(Y,Z)$ up to an $O(1/\sqrt{n})$ additive error.
	Notice that:
	\begin{align*}
		I(A;Z) &= H(A) + H(Z) - H(A, Z) \\
		I(A;Y) &= H(A) + H(Y) - H(A, Y) \\
		H(Y\given Z) &=  H(Y,Z) - H(Y)
	\end{align*}
	Therefore, $I(A;Z), I(A;Y), H(Y\given Z)$ can also be estimated with an $O(1/\sqrt{n})$ additive error. By our assumption, the denominators $I(A;Y), H(A) $ are bound away from zero, therefore simply plug-in their estimator for $H(Y), H(A), H(Z), H(A,Z), H(A,Y), H(Y,Z)$ provides the desired estimator.
\end{proof}

\section{Missing Proofs in Regression (Section~\ref{sec:continuous})}
\label{app:regression}
\subsection{Convexity of $\cregion$}
Analogous to the classification setting, here we first show that the feasible region $\cregion$ is convex:
\rconvex*
\begin{proof}
Let $Z_i = g_i(X)$ for $i\in\{0, 1\}$ with corresponding points $(\Var\Exp[Y\mid Z_i], \Var\Exp[A\mid Z_i])\in\cregion$. Then it suffices if we could show for $\forall u\in[0, 1]$, $(u\Var\Exp[Y\mid Z_0] + (1-u)\Var\Exp[Y\mid Z_1], u\Var\Exp[A\mid Z_0]) + (1-u)\Var\Exp[A\mid Z_1])\in\cregion$ as well.

We give a constructive proof. Due to the symmetry between $A$ and $Y$, we will only prove the result for $Y$, and the same analysis could be directly applied to $A$ as well. For any $u\in[0, 1]$, let $S\sim U(0,1)$, the uniform distribution over $(0,1)$, such that $S\perp (Y, A)$. Consider the following randomized transformation $Z= (Z', S)$, where:
\begin{equation}
    Z' = \begin{cases}
    Z_0 & \text{If } S \leq u, \\
    Z_1 & \text{otherwise}.
    \end{cases}
\end{equation}
To compute $\Var\Exp[Y\mid Z]$, define $K\defeq \Exp[Y\mid Z]$, then by the law of total variance, we have:
\begin{equation*}
    \Var\Exp[Y\mid Z] = \Var(K) = \Exp[\Var(K\mid S)] + \Var\Exp[K\mid S].
\end{equation*}
We first compute $\Var\Exp[K\mid S]$:
\begin{align*}
    \Var\Exp[K\mid S] &= \Var\Exp[\Exp[Y\mid Z]\mid S] \\
                      &= \Var\Exp[Y\mid S] \tag{The law of total expectation} \\
                      &= \Var\Exp[Y] && \tag{$Y\perp S$} \\
                      &= 0.
\end{align*}
On the other hand, for $\Exp[\Var(K\mid S)]$, we have:
\begin{align*}
    \Exp[\Var(K\mid S)] &= \Pr(S = 0)\cdot\Var(K\mid S = 0) + \Pr(S = 1)\cdot \Var(K\mid S = 1) \\
                        &= u\cdot \Var(K\mid S = 0) + (1 - u)\cdot \Var(K\mid S = 1) \\
                        &= u\cdot \Var\Exp[Y\mid Z_0] + (1-u)\cdot \Var\Exp[Y\mid Z_1].
\end{align*}        
Combining both equations above yields:
\begin{equation*}
    \Var\Exp[Y\mid Z] = u\cdot \Var\Exp[Y\mid Z_0] + (1-u)\cdot \Var\Exp[Y\mid Z_1].
\end{equation*}
Similar argument could be used to show that $\Var\Exp[A\mid Z] = u\cdot \Var\Exp[A\mid Z_0] + (1-u)\cdot \Var\Exp[A\mid Z_1]$. So by construction we now find a randomized transformation $Z = g(X)$ such that $(u\cdot \Var\Exp[Y\mid Z_0] + (1-u)\cdot \Var\Exp[Y\mid Z_1], u\cdot \Var\Exp[A\mid Z_0] + (1-u)\cdot \Var\Exp[A\mid Z_1])\in\cregion$, which completes the proof.
\end{proof}

\subsection{Proof of Theorem~\ref{thm:maxvar}}
In this section, we will prove Theorem~\ref{thm:maxvar}. We will provide a proof in a generalized noisy setting, i.e., we no longer assume the noiseless condition (Assumption~\ref{assm:continuous:perfect}) so that the corresponding theorems in the noiseless setting follow as a special case. To this end, we first re-define 
\begin{align}
	 f_Y^*(X) &\defeq \E[Y\mid X]\\
	 f_A^*(X) &\defeq \E[A\mid X]
\end{align}
and $f_Y^*, f_A^* \in \mathbb{H}$. We reuse the notations $a, y$ to denote 
\begin{align}
f_Y^*(X) &= \E[Y|X] = \langle y, \phi(X) \rangle\\
f_A^*(X) &= \E[A|X] = \langle a, \phi(X) \rangle.
\end{align}
It is easy to see that the noiseless setting is indeed a special case where $Y = \E[Y|X], A = \E[A|X]$ almost surely.

For readers' convenience, we restate the Theorem \ref{thm:maxvar} below:
\maxvar*
The following theorem is the generalized version of Theorem \ref{thm:maxvar} in noisy setting:
\begin{restatable}{theorem}{maxvarnoisy}
	The optimal solution of optimization problem~\eqref{equ:maxvar} is upper bounded by
	\begin{equation}
	\underset{Z, \Var\Exp[A\mid Z] = 0}{\max} \Var\Exp[Y\mid Z] \le \Var\E[Y|X] - \frac{\Cov^2(\Exp[Y\mid X], \Exp[A\mid X])}{\Var\Exp[A\mid X]}.
	\end{equation}
	\label{thm:maxvar_noisy}
\end{restatable}

It is easy to see Theorem~\ref{thm:maxvar} is an immediate corollary of this result: under the noiseless assumption, we have $\Var\E[Y|X] = \Var(Y)$ and $\Var\E[A|X] = \Var(A)$.
\begin{proof}
	
 	Using the law of total expectation, 
	\begin{align*}
	\E[Y\mid Z] = \Exp\left[\Exp[Y\mid X]\mid Z\right] = \int_\xxspace \E[Y\mid X,Z]\cdot p(X\mid Z)~dX.
	\end{align*}
	Since $Z=g(X)$ is a function of $X$, we have $Z \perp Y \mid X$, so $ \E[Y\mid X,Z] = \E[Y\mid X] = f_Y^*(X)$. Therefore,
	\begin{align*}
	\E[Y|Z]& = \int_\xxspace \E[Y\mid X,Z]\cdot p(X\mid Z)~dX \\
	&= \int_\xxspace f_Y^*(X)\cdot p(X\mid Z)~dX \\
	&= \E[f_Y^*(X) \mid Z].
	\end{align*}
	Hence,
	\begin{align}\label{eq:simplify_z}
	\Var\E[Y\mid Z] = \Var\E[f_Y^*(X) \mid Z].
	\end{align}
	Therefore, 
	\begin{align*}
	\Var\Exp[Y\mid Z] &= \Var\E[f_Y^*(X) \mid Z] \\
	&= \Var\Exp[\langle y, \phi(X)\rangle\mid Z] \\
	&= \Var\langle y, \Exp[\phi(X)\mid Z]\rangle && \text{(Linearity of Expectation)}\\
	&= \langle y, \Var\Exp[\phi(X)\mid Z]y\rangle.
	\end{align*}
	Similarly, for $A = \langle a, \phi(X)\rangle$, we have:
	\begin{equation*}
	\Var\Exp[A\mid Z] = \langle a, \Var\Exp[\phi(X)\mid Z] a\rangle.
	\end{equation*}
	To simplify the notation, define $V\defeq \Var\Exp[\phi(X)\mid Z]$. Then again, by the law of total variance, it is easy to verify that $0\preceq V\preceq \Sigma = \Var\phi(X)$. Hence the original maximization problem could be 
	relaxed as follows:
	\begin{equation*}
	\underset{V}{\text{maximize}}\quad \langle y, Vy\rangle,\qquad
	\text{subject to}\quad 0 \preceq V\preceq \Sigma, \quad\langle a, Va\rangle = 0.
	\end{equation*}
	We now apply the transformation $V'\defeq \Sigma^{-1/2}V\Sigma^{-1/2}$, $y'\defeq \Sigma^{1/2}y$ and $a'\defeq \Sigma^{1/2}a$ to further simplify the above optimization formulation:
	\begin{equation*}
	\underset{V'}{\text{maximize}}\quad \langle y', V'y'\rangle,\qquad
	\text{subject to}\quad 0 \preceq V'\preceq I, \quad\langle a', V'a'\rangle = 0.
	\end{equation*}
	To proceed, we first decompose $y'$ orthogonally w.r.t.\ $a'$:
	\begin{equation*}
	y' = y'^{\perp a'} + y'^{\parallel a'},
	\end{equation*}
	where $y'^{\perp a'}$ is the component of $y'$ that is perpendicular to $a'$ and $y'^{\parallel a'}$ is the parallel component of $y'$ to $a'$. Using this orthogonal decomposition, we have $\forall V'\preceq I$:
	\begin{align}
	\langle y', V'y'\rangle &= \langle (y'^{\perp a'} + y'^{\parallel a'}), V'(y'^{\perp a'} + y'^{\parallel a'})\rangle\label{equ:holds} \\
		&= \langle y'^{\perp a'} , V'y^{\perp a'} \rangle + \langle y'^{\parallel a'} , V'y^{\parallel a'}\rangle + 2\langle y'^{\perp a'} , V'y^{\parallel a'} \rangle\nonumber  \\
    	&= \langle y'^{\perp a'} , V'y^{\perp a'} \rangle \tag{$V'^{1/2}y'^{\parallel a'} = 0$ since $V'^{1/2}a' = 0$}\nonumber\\
    	&\leq \langle y'^{\perp a'}, y'^{\perp a'} \rangle, \tag{$V'\preceq I$} 
	\end{align}
    Next, realize that the vector $y'^{\perp a'}$ could be constructed as follows:
	\begin{equation*}
	y'^{\perp a'} = y' - y'^{\parallel a'} = y' - \frac{\langle y', a'\rangle}{\langle a', a'\rangle}a', 
	\end{equation*}
	We then plug the above expression of $y'^{\perp a'}$ to~\eqref{equ:holds}:
	\begin{align*}
	\langle y'^{\perp a'}, y'^{\perp a'} \rangle &= \bigg\langle y' - \frac{\langle y', a'\rangle}{\langle a', a'\rangle}a', y' - \frac{\langle y', a'\rangle}{\langle a', a'\rangle}a'\bigg\rangle \\
	&= \langle y', y'\rangle - 2\frac{\langle y', a'\rangle}{\langle a', a'\rangle}\langle y', a'\rangle + \frac{\langle y', a'\rangle^2}{\langle a', a'\rangle^2}\langle a', a'\rangle \\
	&= \langle y', y'\rangle - \frac{\langle y', a'\rangle^2}{\langle a', a'\rangle} \\
	&= \langle y, \Sigma y\rangle - \frac{\langle y, \Sigma a\rangle^2}{\langle a, \Sigma a\rangle} \\
	&= \Var\Exp[Y\mid X] - \frac{\Cov^2(\Exp[Y\mid X], \Exp[A\mid X])}{\Var\Exp[A\mid X]} \\
	&= \Var\Exp[Y\mid X]\bigg(1 - \frac{\Cov^2(\Exp[Y\mid X], \Exp[A\mid X])}{\Var\Exp[Y\mid X]\cdot \Var\Exp[A\mid X]}\bigg) \\
	&= \Var\Exp[Y\mid X]\cdot(1 - \rho^2_{YA}),
	\end{align*}
	by using the fact that $ \Var\E[Y\mid X]= \langle y, \Sigma y\rangle$, $ \Var\E[A\mid X] = \langle a, \Sigma a\rangle$ and $\Cov(\E[Y\mid X], \E[A\mid X]) = \langle y, \Sigma a\rangle$.
    
	To complete the proof, we only need to verify that the inequality in~\eqref{equ:holds} can indeed be attained. We do so by construction, let 
	\begin{equation*}
	    V'^{*} \defeq I - a'_0\otimes a'_0,
	\end{equation*}
	where $a'_0 \defeq a' / \|a'\|$ is the unit vector of $a'$. Clearly, $0 \preceq V'^{*}\preceq I$. Furthermore, 
	\begin{align*}
	    \langle y'^{\perp a'} , V'^{*}y^{\perp a'} \rangle - \langle y'^{\perp a'}, y'^{\perp a'} \rangle &= \langle y'^{\perp a'} , (a'_0\otimes a'_0) y^{\perp a'} \rangle \\
	    &= \langle y'^{\perp a'} , a'_0\rangle^2 \\
	    &= \langle y'^{\perp a'} , a'\rangle^2 / \|a'\|^2 \\
	    &= 0,
	\end{align*}
	completing the proof.
\end{proof}

\subsection{Proof of Theorem~\ref{thm:minvar}}
Again, we will provide a proof of Theorem~\ref{thm:minvar} in a generalized noisy setting, i.e., we no longer assume the noiseless condition (Assumption~\ref{assm:continuous:perfect}) so that the corresponding theorems in the noiseless setting follow as a special case. We first restate Theorem \ref{thm:minvar} below:
\minvar*
The following theorem is the generalized version of Theorem \ref{thm:minvar} in noisy setting:
\begin{restatable}{theorem}{minvarnoisy}
	The optimal solution of optimization problem~\eqref{equ:minvar} is lower bounded by
	\begin{equation}
	\underset{Z, \Var\Exp[Y\mid Z] = \Var\E[Y|X]}{\min}~\Var\Exp[A\mid Z] \ge \frac{\Var\E[Y|X]\cdot \langle a, y\rangle^2}{\langle y, y\rangle^2}.
	\end{equation}
	\label{thm:minvar_noisy}
\end{restatable}

It is easy to see Theorem~\ref{thm:minvar} is an immediate corollary of this result: under the noiseless assumption, we have $\Var\E[Y|X] = \Var(Y)$ and $\Var\E[A|X] = \Var(A)$.
\begin{proof}
As in the proof of Theorem~\ref{thm:maxvar}, we have the following identities hold:
	\begin{align*}
	\Var\Exp[A\mid Z] &= \langle a, \Var\Exp[\phi(X)\mid Z] a\rangle,\\
	\Var\Exp[Y\mid Z] &= \langle y, \Var\Exp[\phi(X)\mid Z] y\rangle.
	\end{align*}
	Again, let $V\defeq \Var\Exp[\phi(X)\mid Z]$ so that we can 
	relax the optimization problem as follows:
	\begin{equation*}
	\begin{aligned}
		\underset{V}{\text{minimize}}\quad & \langle a, Va\rangle,\\
    	\text{subject to}\quad & 0  \preceq V\preceq \Sigma, \\
    	                         & \langle y, Vy\rangle = \Var\E[Y|X] = \langle y, \Sigma y\rangle.
	\end{aligned}
	\end{equation*}
	Again, we apply the transformation $V'\defeq \Sigma^{-1/2}V\Sigma^{-1/2}$, $y'\defeq \Sigma^{1/2}y$ and $a'\defeq \Sigma^{1/2}a$ to further simplify the above optimization formulation:
	\begin{equation*}
	\begin{aligned}
		\underset{V'}{\text{minimize}}\quad & \langle a', V'a'\rangle,\\
    	\text{subject to}\quad & 0  \preceq V'\preceq I, \\
    	                         & \langle y', V'y'\rangle = \langle y', y'\rangle.
    \end{aligned}
	\end{equation*}
	Note that since by the constraint $\langle y', V'y'\rangle = \langle y', y'\rangle$, which means $\langle y', (V' - I)y'\rangle = 0$, we also have $(V' - I)^{1/2} y' = 0$. To proceed, we first decompose $a'$ orthogonally w.r.t.\ $y'$:
	\begin{equation*}
	a' = a'^{\perp y'} + a'^{\parallel y'},
	\end{equation*}
	where $a'^{\perp y'}$ is the component of $a'$ that is perpendicular to $y'$ and $a'^{\parallel y'}$ is the parallel component of $a'$ to $y'$. Using this orthogonal decomposition, we have $\forall V'\succeq 0$:
    \begin{align}
        \langle a', V'a'\rangle &= \langle a'^{\perp y'} + a'^{\parallel y'}, V'(a'^{\perp y'} + a'^{\parallel y'})\rangle \label{equ:holds2}\\
        &= \langle a'^{\perp y'}, V' a'^{\perp y'}\rangle + \langle a'^{\parallel y'}, V' a'^{\parallel y'}\rangle + 2 \langle a'^{\perp y'}, V' a'^{\parallel y'}\rangle \nonumber\\
        &= \langle a'^{\perp y'}, V' a'^{\perp y'}\rangle + \langle a'^{\parallel y'}, V' a'^{\parallel y'}\rangle + 2 \langle a'^{\perp y'}, V' a'^{\parallel y'}\rangle - 2 \langle a'^{\perp y'}, a'^{\parallel y'}\rangle \nonumber\\
        &= \langle a'^{\perp y'}, V' a'^{\perp y'}\rangle + \langle a'^{\parallel y'}, V' a'^{\parallel y'}\rangle + 2 \langle a'^{\perp y'}, (V' - I) a'^{\parallel y'}\rangle \nonumber\\
        &= \langle a'^{\perp y'}, V' a'^{\perp y'}\rangle + \langle a'^{\parallel y'}, V' a'^{\parallel y'}\rangle \tag{$(V' - I)^{1/2} a'^{\parallel y'} = 0$ since $(V' - I)^{1/2} y' = 0$}\nonumber\\
        &\geq \langle a'^{\parallel y'}, V' a'^{\parallel y'}\rangle \tag{$V'\succeq 0$}\nonumber
    \end{align}	
	On the other hand, from linear algebra, it is clear that
	\begin{equation*}
	a'^{\parallel y'} = \frac{\langle a', y'\rangle}{\langle y', y' \rangle}y'.
	\end{equation*}
	Plug the above formula to~\eqref{equ:holds2}, yielding: 
	\begin{align*}
        \langle a'^{\parallel y'}, V' a'^{\parallel y'}\rangle  &= \bigg\langle \frac{\langle a', y'\rangle}{\langle y', y' \rangle}y', V' \frac{\langle a', y'\rangle}{\langle y', y' \rangle}y'\bigg\rangle \\
        &= \frac{\langle a', y'\rangle^2}{\langle y', y' \rangle^2}\langle y', V'y'\rangle \\
        &= \frac{\langle a', y'\rangle^2}{\langle y', y' \rangle^2}\langle y', y'\rangle \tag{$\langle y', V'y'\rangle = \langle y', y'\rangle$} \\
        &= \frac{\langle a, \Sigma y\rangle^2}{\langle y, Vy \rangle} \\
        &= \frac{\Cov^2(\Exp[Y\mid X], \Exp[A\mid X])}{\Var\Exp[Y\mid X]} \\
        &= \Var\Exp[A\mid X]\cdot \rho_{YA}^2,
    \end{align*}
	where we use the fact that $ \Var\E[Y\mid X]= \langle y, \Sigma y\rangle$, $ \Var\E[A\mid X] = \langle a, \Sigma a\rangle$ and $\Cov(\E[Y\mid X], \E[A\mid X]) = \langle y, \Sigma a\rangle$.
	
	To complete the proof, we only need to verify that the inequality in~\eqref{equ:holds2} can indeed be attained. We do so by construction, let 
	\begin{equation*}
	    V'^{*} \defeq y'_0\otimes y'_0,
	\end{equation*}
	where $y'_0 \defeq y' / \|y'\|$ is the unit vector of $y'$. Clearly, $0 \preceq V'^{*}\preceq I$ and $\langle y', y'\rangle = \langle y', V'^{*}y'\rangle$. Furthermore, 
	\begin{align*}
	   \langle a'^{\perp y'}, V'^{*} a'^{\perp y'}\rangle &= \langle a'^{\perp y'}, (y'_0\otimes y'_0) a'^{\perp y'}\rangle \\
	   &= \langle a'^{\perp y'}, y'_0 \rangle^2 \\
	   &= \langle a'^{\perp y'}, y' \rangle^2 / \|y'\|^2 \\
	   &= 0,
	\end{align*}
	completing the proof.
\end{proof}
Next, we provide the proof for Corollary~\ref{cor:continuous:bounds}:
\contbounds*
\begin{proof}
For the first inequality, let $\ell(\cdot, \cdot)$ be the squared loss. If $\Var\Exp[A\mid Z] = 0$, then by Theorem~\ref{thm:maxvar}
\begin{align*}
    \inf_h~\Exp[\ell(Y, h(Z))] 
    &= \Exp\Var(Y\mid Z) \\
    &= \Var(Y) - \Var\Exp[Y\mid Z] \\
    &\geq \Var(Y) - \left\{\Var(Y)\cdot (1-\rho_{YA}^2)\right\} \\
    &= \Var(Y)\cdot \rho_{YA}^2.
\end{align*}
Similarly, if the feature $Z$ preserves all the information w.r.t.\ $Y$, e.g., $\Var\Exp[Y\mid Z] = \Var(Y)$, then Theorem~\ref{thm:minvar} immediately implies the following bound, 
\begin{align*}
    \inf_h~\Exp[\ell(A, h(Z))] &= \Exp\Var(A\mid Z) \\
    &= \Var(A) - \Var\Exp[A\mid Z] \\
    &\leq \Var(A) - \Var(A)\cdot \rho_{YA}^2 \\
    &= \Var(A)\cdot (1 - \rho_{YA}^2).
\end{align*}
\end{proof}

\subsection{Proof of Theorem~\ref{thm:lower_bound}}

\noindent

We will prove a generalized version of Theorem \ref{thm:lower_bound} without noiseless assumption, stated below:

\begin{restatable}{theorem}{lagrangian}
The optimal solution of the Lagrangian has the following lower bound:
\begin{equation*}
\opt(\lambda) \geq \frac12\Big\{\lambda \Var\E[A| X] - \Var\E[Y| X] - \sqrt{(\lambda \Var\E[A| X] + \Var\E[Y| X])^2 - 4\lambda\Cov^2(\E[Y| X], \E[A| X])}\Big\} .
\label{equ:lag_noisy}
\end{equation*}
\label{thm:lower_bound_noisy}
\end{restatable}
To have a sanity check, note that when the noiseless assumption holds, we have $\Var\E[A|X] = \Var(A)$ and $\Var\E[Y|X] = \Var(Y)$, hence the bound above simplifies to:
	\begin{equation*}
	    \frac12\Big\{\lambda \Var(A) - \Var(Y) - \sqrt{(\lambda \Var(A) +\Var(Y))^2 - 4\lambda\cdot \rho_{YA}^2 \Var(A)\Var(Y)}\Big\}.
	\end{equation*}
	which reduces to the lower bound in Theorem \ref{thm:lower_bound}.

\begin{proof}[Proof of Theorem \ref{thm:lower_bound}]

Recall that our objective is:
\begin{equation}\label{eq:opt_l2_general_simplified}
\underset{Z=g(X)}{\text{minimize}}  \quad \lambda\Var\E[A\mid Z] - \Var\E[Y\mid Z]
\end{equation}
As in the proof of Theorem~\ref{thm:maxvar}, we have the following identities hold:
\begin{align*}
\Var\Exp[A\mid Z] &= \langle a, \Var\Exp[\phi(X)\mid Z] a\rangle,\\
\Var\Exp[Y\mid Z] &= \langle y, \Var\Exp[\phi(X)\mid Z] y\rangle.
\end{align*}
Again, let $V\defeq \Var\Exp[\phi(X)\mid Z]$ so that we can relax the optimization problem as follows:
\begin{equation}
\begin{aligned}
	\underset{V}{\text{minimize}}\quad & \lambda\langle a, Va\rangle - \langle y, Vy\rangle = \tr(V(\lambda aa^T - yy^T)),\\
	\text{subject to}\quad & 0  \preceq V\preceq \Sigma.
\end{aligned}
\label{equ:relaxed_overall}
\end{equation}
It is worth pointing out that the constraint in the optimization problem~\eqref{equ:relaxed_overall} is a superset (hence a relaxation) of the one in~\eqref{eq:opt_l2_general_simplified}, since there may not exist representation $Z = g(X)$ that attain a given $V = \Var\Exp[\phi(X)\mid Z]$. As a result, the optimal solution of~\eqref{equ:relaxed_overall} is a lower bound of that of~\eqref{eq:opt_l2_general_simplified}.

Similar to the proofs of Theorem~\ref{thm:maxvar} and Theorem~\ref{thm:minvar}, we apply the transformation $V'\defeq \Sigma^{-1/2}V\Sigma^{-1/2}$, $y'\defeq \Sigma^{1/2}y$ and $a'\defeq \Sigma^{1/2}a$ to further simplify the above optimization formulation:
\begin{equation}
\label{equ:relaxed_simplified_overall}
\begin{aligned}
	\underset{V'}{\text{minimize}}\quad & \lambda\langle a', V'a'\rangle - \langle y', V'y'\rangle = \tr(V'(\lambda a'a'^T - y'y'^T)),\\
	\text{subject to}\quad & 0  \preceq V'\preceq I.
\end{aligned}
\end{equation}
One key observation to the optimization problem~\eqref{equ:relaxed_simplified_overall} is that the matrix $R\defeq \lambda a'a'^T - y'y'^T$ has rank at most 2. Furthermore, thanks to Lemma \ref{lem:eigen}, we can fully characterize the spectrum of this matrix: Let $\sigma_i(R)$ be the $i$-th largest eigenvalue of $R$, then by Lemma~\ref{lem:eigen}, the spectrum of $R$ admits the following form:
\begin{align*}
\sigma_1(R) &= \frac12\Big\{\lambda \ip{a, \Sigma a} - \ip{y, \Sigma y} + \sqrt{(\lambda \ip{a, \Sigma a} + \ip{y, \Sigma y})^2 - 4\lambda \ip{a, \Sigma y}^2}\Big\}, \\
\sigma_d(R) &= \frac12\Big\{\lambda \ip{a, \Sigma a} - \ip{y, \Sigma y} - \sqrt{(\lambda \ip{a, \Sigma a} + \ip{y, \Sigma y})^2 - 4\lambda \ip{a, \Sigma y}^2}\Big\}, \\
\sigma_2(R) &= \cdots = \sigma_{d-1}(R) =0.
\end{align*}
Now by the spectral theorem, $R$ could be decomposed as follows:
\begin{equation*}
    R = \sigma_1(R)u_1u_1^T + \sigma_d(R) u_du_d^T,
\end{equation*}
where $u_i$ is the corresponding orthogonal eigenvectors. Hence, due to the Von-Neumann's trace inequality, we have:
\begin{align*}
    \tr(V'(\lambda a'a'^T - y'y'^T)) &= \tr(V'R) \\
    &  \geq \sum_{i=1}^d \sigma_i(R)\sigma_{d-i+1}(V') \tag{Von-Neumann's trace inequality} \\
    &= \sigma_1(R)\sigma_d(V') + \sigma_d(R)\sigma_1(V') \tag{$\sigma_2(R) = ... =\sigma_{d-1}(R)=0$} \\
    & \geq \sigma_d(R)\sigma_1(V') \tag{$\sigma_1(R) > 0$ and $\sigma_d(V') \geq 0$} \\
    & \geq \sigma_d(R) \tag{$\sigma_d(R) < 0$ and $\sigma_1(V') \leq 1$}.
\end{align*}
Now it suffices to verify that the lower bound above is attainable. To see this, let $V'^* = u_du_d^T$, i.e., the projection matrix of the eigenvector corresponding to $\sigma_d(R)$. The last step is to plug $\sigma_d(R)$ into the original optimization objective~\eqref{eq:opt_l2_general_simplified}, yielding:
\begin{align*}
    \opt(\lambda) &= \min_{Z}~ \lambda \Var\Exp[A\mid Z] - \Var\Exp[Y\mid Z]  \\
                  &\geq \ \sigma_d(R) \\
                  &=  \frac12\Big\{\lambda \ip{a, \Sigma a} - \ip{y, \Sigma y} - \sqrt{(\lambda \ip{a, \Sigma a} + \ip{y, \Sigma y})^2 - 4\lambda \ip{a, \Sigma y}^2}\Big\} \\
                  &= \frac12\Big\{\lambda \Var\E[A| X] - \Var\E[Y| X] - \sqrt{(\lambda \Var\E[A| X] + \Var\E[Y| X])^2 - 4\lambda\Cov^2(\E[Y| X], \E[A| X])}\Big\} 
\end{align*}        
Hence we have completed the proof.
\end{proof}

\subsection{Explicit formula for eigenvalues}\label{sec:eigenvalues}

The following lemma is used the in the proof of Theorem~\ref{thm:lower_bound} to characterize the spectrum of the matrix $R = \lambda a'a'^T - y'y'^T$.

\begin{lemma}\label{lem:eigen}
Let $R \defeq \lambda a'a'^T - y'y'^T$, where $y'\defeq \Sigma^{1/2}y$ and $a'\defeq \Sigma^{1/2}a$. Then the eigenvalues of $R$ are 
\begin{align*}
\sigma_1(R) &= \frac12\Big\{\lambda \ip{a, \Sigma a} - \ip{y, \Sigma y} + \sqrt{(\lambda \ip{a, \Sigma a} + \ip{y, \Sigma y})^2 - 4\lambda \ip{a, \Sigma y}^2}\Big\}, \\
\sigma_d(R) &= \frac12\Big\{\lambda \ip{a, \Sigma a} - \ip{y, \Sigma y} - \sqrt{(\lambda \ip{a, \Sigma a} + \ip{y, \Sigma y})^2 - 4\lambda \ip{a, \Sigma y}^2}\Big\}\\
\sigma_2(R) &= \cdots = \sigma_{d-1}(R) =0
\end{align*}
\end{lemma}
\begin{proof}
Since $\rank(R) \leq \rank(\lambda a'a'^T - y'y'^T) \leq 2$, $R$ has at most two non-zero eigenvalues $\sigma_1(R)$ and $\sigma_{d}(R)$. Notice that 
\begin{align*}
\tr(R) &= \sum_{i=1}^d \sigma_i(R) = \sigma_1(R) + \sigma_{d}(R), \\
\tr(R^2) &= \sum_{i=1}^d \sigma^2_i(R) = \sigma_1^2(R) + \sigma_{d}^2(R)
\end{align*}
We can write $\tr(R)$ and $\tr(R^2)$ explicitly:
\begin{equation*}
\tr(R) = \lambda \tr (a'a'^T) - \tr (y'y'^T) = \lambda \ip{a, \Sigma a} - \ip{y, \Sigma y}
\end{equation*}
\begin{align*}
\tr(R^2) &= \tr((\lambda a'a'^T - y'y'^T) (\lambda a'a'^T - y'y'^T)) \\
         &= \lambda^2 (a'^Ta')\tr(a'a'^T) - \lambda (a'^Ty')\tr(a'y'^T) - \lambda (a'^Ty')\tr(y'a'^T) + (y'^Ty')\tr(y'y'^T)\\
         &= \lambda^2 (a'^Ta')^2 - 2\lambda (a'^Ty')^2 + (y'^Ty')^2 \\
         &= \lambda^2 (a^T\Sigma a)^2 - 2\lambda (a^T\Sigma y)^2 + (y^T\Sigma y)^2.
\end{align*}
Therefore, 
\begin{align*}
\sigma_1(R) + \sigma_{d}(R) &=  \lambda \ip{a, \Sigma a} - \ip{y, \Sigma y} \\
\sigma_1(R) \sigma_{d}(R) &= \frac12 \left(
(\sigma_1(R) + \sigma_{d}(R))^2 -
(\sigma_1^2(R) + \sigma_{d}^2(R)) 
\right)\\
&= \lambda \ip{a, \Sigma y}^2 - \lambda \ip{a, \Sigma a} \ip{y, \Sigma y}
\end{align*}	
Thus $\sigma_1(R)$ and $\sigma_d(R)$ are the roots of the quadratic equation in terms of $x$:
\begin{equation*}
x^2 - ( \lambda \ip{a, \Sigma a} - \ip{y, \Sigma y}) x + \lambda \ip{a, \Sigma y}^2 - \lambda \ip{a, \Sigma a} \ip{y, \Sigma y} = 0
\end{equation*}
We complete the proof by solving this quadratic equation.
\end{proof}

\subsection{Proof of Theorem \ref{thm:constrained_upper}}
We consider the constrained problem:
\begin{equation}
\begin{aligned}
& \underset{Z}{\text{maximize}} && \Var\Exp[Y\mid Z] \\
& \text{subject to} && \Var\Exp[A\mid Z] \le c.
\end{aligned}
\label{equ:maxvar:relax}
\end{equation}
Again, in what follows we provide proof for a generalized theorem without the noiseless assumption. One can readily verify that $\Var\E[Y|X] = \Var(Y)$ and $\Var\E[A|X] = \Var(A)$ under the noiseless setting. Hence the theorem below reduces to Theorem~\ref{thm:constrained_upper} when the noiseless assumption holds.
\begin{theorem}
	The optimal solution of the constrained problem \eqref{equ:mmaxvar:relax} has the following upper bound: for any $c \in [0, \rho_{YA}^2 \Var\Exp[A\mid X]]$, we have
	\begin{align}
	\Var\Exp[Y\mid Z] \le \Var\E[Y| X]\left(  2 \rho_{YA}\sqrt{(1-\rho_{YA}^2)\alpha(1-\alpha)}+1-\alpha - \rho_{YA}^2 + 2\alpha \rho_{YA}^2 \right) 
	\end{align}
	where $\alpha \defeq c/\Var\Exp[A\mid X]$.
\end{theorem}
\begin{proof}
	Recall that 
	\begin{equation}
	\Var \E[Y|Z] \leq - \sup_{\lambda \ge 0} \left\{ \opt(\lambda) - \lambda c \right\}
	\end{equation}
	Therefore, it is crucial to find the superimum on the RHS, for which we provide the following lemma:
	%
	
	\begin{lemma}\label{lem:inner_sup}
		Define
		\begin{equation}
		\psi(\lambda;a,b,c, t) :=  \lambda c - \sqrt{a^2+\lambda^2 b^2 + 2\lambda t ab}
		\end{equation}
		where $a,b\ge 0, |c| \leq b, |t| \leq 1$
		Then, 
		\begin{equation}\label{eqn:sup_expression}
		\sup_{\lambda \ge 0} \psi(\lambda;a,b,c, t) = \begin{cases}
		-\frac{a}{b}\left(\sqrt{(1-t^2)(b^2-c^2)} +ct\right) & \text{ if }  c \ge t \cdot b, \\
		-a & \text{ else }.
		\end{cases}
		\end{equation}
	\end{lemma}	
	\begin{proof} To avoid notational clutter, we use $\psi(\lambda)$ as a shorthand for $\psi(\lambda;a,b,c, t) $ when the meaning is clear from context. Direct calculation gives the derivative of $\psi$:
		\begin{equation}
		\frac{\partial \psi}{\partial \lambda} = c - \frac{ b (\lambda b +at)}{\sqrt{a^2+\lambda^2 b^2 + 2\lambda t ab}}
		\end{equation}
		When $c = b$, the derivative $\frac{\partial \psi}{\partial \lambda}$ is always positive. Therefore, the superimum is achieved at $ \psi(+\infty) = -at$, which coincides with the first case in equation \eqref{eqn:sup_expression}.
		
		Solving the equation $\frac{\partial \psi}{\partial \lambda} = 0$ gives two real roots:
		\begin{equation}
		\lambda^* = -\frac{at}{b} \pm \frac{ac}{b} \sqrt{\frac{1-t^2}{b^2-c^2}}
		\end{equation}
		When $b > c \ge t \cdot b$, one of the roots $\lambda^* = -\frac{at}{b} +  \frac{ac}{b} \sqrt{\frac{1-t^2}{b^2-c^2}} $ is positive.  Hence, the superimum is achieved at $\psi(0), \psi(\lambda^*)$ or $\psi(+\infty)$. We can verify that $\psi(\lambda^*) = -\frac{a}{b}\left(\sqrt{(1-t^2)(b^2-c^2)} +ct\right) $ is indeed the maximum among them.
		
		When $c \le t \cdot b$, there is no stationary point in $[0, +\infty)$, therefore the superimum is achieved at $\psi(0)$ or $\psi(+\infty)$. We can verify that $\psi(0) = -a$ is the larger one. 
		
		Combining the three cases, we have completed the proof.
	\end{proof}
	
	Using the definition of $\psi(\lambda; a,b,c,t)$, it is now clear that
	\begin{align*}
	-\sup_{\lambda \ge 0} \Big\{\opt(\lambda) - \lambda c \Big\} &= -\frac12 \sup_{\lambda \ge 0}  \psi(\lambda; a', b',c', t') + \frac12  \Var\E[Y| X]
	\end{align*}
	where 
	\begin{equation}
	a' = \Var\E[Y| X],  b' = \Var\E[A| X], c' =\Var\E[A| X] - 2 c, t' = 1-2 \rho_{YA}^2
	\end{equation}
	To simplify the equations, we introduce the notation 
	\begin{equation}
	\alpha = \frac{c}{\Var\E[A| X]}
	\end{equation}
	Notice that $c' - t'b' = 2(\rho_{YA}^2 \Var\E[A| X] - c )= 2\Var\E[A| X](\rho_{YA}^2  - \alpha ) $, by Lemma \ref{lem:inner_sup}, we have
	
	\textbf{(1) When $\alpha > \rho_{YA}^2 $: } we have $c' - t'b' <0$, therefore,
	\begin{align*}
	-\sup_{\lambda \ge 0}  \psi(\lambda; a', b',c', t') = a' =  \Var\E[Y| X]
	\end{align*}
	hence, in this case,  we have 
	\begin{equation}
	\Var\E[Y| Z] \le- \sup_{\lambda \geq 0}\Big\{\opt(\lambda)-\lambda c\Big\} = \frac12 \Var\E[Y| X] + \frac12 \Var\E[Y| X] = \Var\E[Y| X]
	\end{equation}
	which is a trivial upper bound. Indeed, in this regime, $ \Var\E[A| Z]$ is allowed to be larger than the lower bound in Theorem \ref{thm:minvar_noisy}, where one achieves no trade-off on the utility $\Var\E[Y| Z]$.
	
	\textbf{(2) When $\alpha \le \rho_{YA}^2 $: } we have $c' - t'b' \ge0$, therefore,
	\begin{align*}
	-\sup_{\lambda \ge 0}  \psi(\lambda; a', b',c', t') &=-\frac{a'}{b'}\left(\sqrt{(1-t'^2)(b'^2-c'^2)} +c't'\right) \\
	&= -\Var\E[Y| X]\left( 4 \rho_{YA}\sqrt{(1-\rho_{YA}^2)\alpha(1-\alpha)} + (1-2\alpha) (1-2\rho_{YA}^2)\right)
	\end{align*}
	hence, in this case,  we have 
	\begin{align*}
	\Var\E[Y| Z] &\le- \sup_{\lambda \geq 0}\Big\{\opt(\lambda)-\lambda c\Big\} \\
	&= -\frac12 \sup_{\lambda \ge 0}  \psi(\lambda; a', b',c', t') + \frac12 \Var\E[Y| X] \\
	&= \Var\E[Y| X]\left(  2 \rho_{YA}\sqrt{(1-\rho_{YA}^2)\alpha(1-\alpha)}+1-\alpha - \rho_{YA}^2 + 2\alpha \rho_{YA}^2 \right) 
	\end{align*}
	Specifically, when $c = 0$, i.e. $\alpha = 0$,  this upper bound simplifies to:
	\begin{align}
	\Var\E[Y| Z] &\leq \Var\E[Y| X] (1-\rho_{YA}^2),
	\end{align}
	which is exactly Theorem \ref{thm:maxvar_noisy}.
\end{proof}

To finish this section, it is also worth pointing out that the lower bound in~\eqref{equ:lag} admits a geometric interpretation, as shown in Figure~\ref{fig:triangle}. Essentially, the lower bound of $\opt(\lambda)$ in Theorem~\ref{thm:lower_bound} corresponds to the half of $|OA| - (|OY| + |AY|)$ in the triangle of Figure~\ref{fig:triangle}. Hence due to the triangle inequality, this term is always nonpositive. In particular, if $c = 0$, then $\sup_{\lambda\geq 0}\opt(\lambda)$ is attained at $\lambda = 0$.
\begin{figure}[tb]
    \centering
    \includegraphics{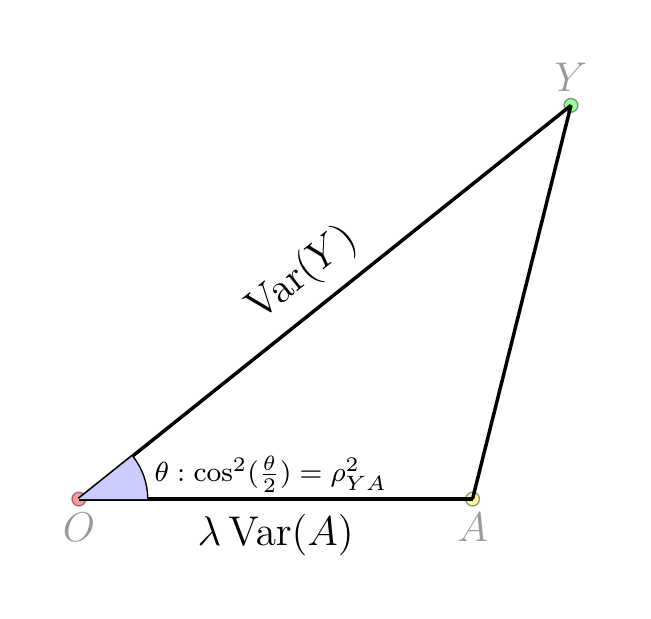}
    \caption{A geometric interpretation of $\opt(\lambda)$. By the cosine theorem, $|AY|^2 = \Var^2(Y) + \lambda^2\Var^2(A) - 2\lambda\Var(A)\Var(Y)(2\rho_{YA}^2 - 1)$. Hence the lower bound of $\opt(\lambda)$ in Theorem~\ref{thm:lower_bound} corresponds to the half of $|OA| - (|OY| + |AY|)$ in the above triangle. Hence due to the triangle inequality, this term is always nonpositive.}
    \label{fig:triangle}
\end{figure}

\subsection{Tightness of the Bounds} 
\label{sec:achievability}
In the main text we mention that a sufficient condition on the regularity of $(X, \phi)$ is that $\phi(X)$ follows a Gaussian distribution. To prove this theorem, we first prove the following lemma, which shows that for rank-1 matrix $M$ with a special structure, it suffices for us to choose $g(X)$ as a deterministic linear transform.
\begin{lemma}
Let $\Sigma = \sum_{i=1}^d \sigma_i u_iu_i^T$ be the spectral decomposition of $\Sigma$. If $\phi(X)$ is Gaussian and $M = \sigma_j u_ju_j^T$ for some $j\in[d]$, then there exists $L = M\Sigma^{-1/2}$ such that for $Z = L\phi(X)$, we have 
\begin{equation*}
    \Var\E[ \phi(X)\mid Z] = M.
\end{equation*}
\label{lemma:rank1}
\end{lemma}
\begin{proof}
	Note that when $\phi(X)$ is Gaussian, $(\phi(X), L \phi(X))$ is jointly Gaussian for any $L \in \RR^{d \times d}$, since $(\phi(X), L\phi(X))$ is a linear transformation of $\phi(X)$. Let $Z = L \phi(X)$. It can be readily verified that the covariance $\Cov(\phi(X), Z)$ is given by
	\begin{align*}
	    \Cov(\phi(X), Z) &= \E[\langle \phi(X) - \E[\phi(X)], Z - \E[Z]\rangle] \\
	    &= \E[\langle \phi(X) - \E[\phi(X)], L \phi(X) - L \E[\phi(X)]\rangle] \\
	    &= \Sigma L^T.
	\end{align*}
	Furthermore, we know that the conditional distribution $\phi(X)\mid Z$ is Gaussian as well, with mean and covariance given by:
	\begin{align*}
	\E[\phi(X)\mid Z] &= \E[\phi(X)] + \Cov(\phi(X), Z) \Var(Z)^{-1}(Z - \E[Z]) \\
	&= \E[\phi(X)] + \Sigma L^T(L\Sigma L^T)^{-1} (Z - L\E[\phi(X)]), \\
	\Var(\phi(X)\mid Z) &= \Var(\phi(X)) - \Cov(\phi(X), Z)\Var(Z)^{-1}\Cov(Z, \phi(X)) \\
	&= \Sigma - \Sigma L^T (L\Sigma L^T)^{-1} L \Sigma.
	\end{align*}
	Hence, 
	\begin{align*}
	\E[\Var(\phi(X)\mid Z)] = \Sigma - \Sigma L^T (L\Sigma L^T)^{-1} L \Sigma,
	\end{align*}
	and by the law of total variance, we have
	\begin{align*}
	    \Var\E[\phi(X)\mid Z] &= \Var(\phi(X)) - \E[\Var(\phi(X)\mid Z)] \\
	    &= \Sigma - \left(\Sigma - \Sigma L^T (L\Sigma L^T)^{-1} L \Sigma\right) \\
	    &= \Sigma L^T (L\Sigma L^T)^{-1} L \Sigma.
	\end{align*}
    Now it suffices that for $M = \sigma_j u_ju_j^T$, we could find the corresponding $L\in\RR^{d\times d}$ such that for $Z = L\phi(X)$, 
	\begin{equation*}
	    M = \Var\E[\phi(X)\mid Z] = \Sigma L^T (L\Sigma L^T)^{-1} L \Sigma.
	\end{equation*}
	To proceed, define $P\defeq L\Sigma^{1/2}$. In order to find the linear transform $L$, we need to solve the following equation in terms of $L$:
    \begin{align*}
        M = \Sigma L^T(L\Sigma L^T)^{\dagger}L\Sigma &\iff \Sigma^{-1/2}M\Sigma^{-1/2} = \Sigma^{1/2}L^T(L\Sigma L^T)^\dagger L\Sigma^{1/2} \\
                                                     &\iff \Sigma^{-1/2}M\Sigma^{-1/2} = P^T(PP^T)^\dagger P \\
                                                     &\iff u_ju_j^T = P^T(PP^T)^\dagger P \\
                                                    &\iff M(M^TM)^\dagger M^T = P^T(PP^T)^\dagger P. 
    \end{align*}	
	It is then straightforward to read the solution from the above equation, giving us $M^T = P$, which implies
	\begin{equation*}
        M = M^T = L\Sigma^{1/2} \implies L = M\Sigma^{-1/2},	    
	\end{equation*}
	completing the proof.
\end{proof}
With the help of Lemma~\ref{lemma:rank1}, we can now prove Theorem~\ref{thm:regular}. The high-level idea of the proof, again, is by constructing randomized feature transformations. 
\gaussian*
\begin{proof}
Recall that by definition of regularity, it suffices for us to prove the following: for any positive semidefinite matrix $M$: $\Sigma \succeq M \succeq 0$, there exists (randomized) $Z = g(X)$, such that 
\begin{equation*}
\Var\E[ \phi(X) \mid Z] = M.
\end{equation*}
Let $\Sigma = \sum_{i=1}^d \sigma_i u_iu_i^T$ be the spectral decomposition of $\Sigma$. Since $0\preceq M\preceq \Sigma$, there exists $0 \leq m_i \leq \sigma_i, \forall i\in[d]$, such that 
\begin{equation*}
    M = \sum_{i=1}^d m_i u_iu_i^T.
\end{equation*}
Define $s = \sum_{i=1}^d m_i / \sigma_i$, then $M$ could be equivalently expressed as
\begin{equation*}
    M = \sum_{i=1}^d m_i u_iu_i^T = \sum_{i=1}^d \frac{m_i}{s\sigma_i}\cdot s\sigma_i u_i u_i^T\eqdef \sum_{i=1}^d \frac{m_i}{s\sigma_i}\cdot sM_i.
\end{equation*}
In other words, $M$ is a convex combination of $\{sM_i\}_{i\in[d]}$. Following Lemma~\ref{lemma:rank1}, for each $M_i = \sigma_i u_iu_i^T$, there exists $L_i \defeq M_i\Sigma^{-1/2}$ and we can construct $Z_i \defeq sL_i\phi(X)$, such that
\begin{equation*}
    \Var\E[ \phi(X)\mid Z_i] = sM_i.
\end{equation*}
Now consider the following randomized feature $Z$. First, let $S \sim U(0, 1)$ be a uniformly random variable over $(0, 1)$. Construct
\begin{equation}
    Z = \begin{cases}
    Z_1 & \text{If } 0 \leq S < \frac{m_1}{s\sigma_1}, \\
    Z_2 & \text{If } \frac{m_1}{s\sigma_1} \leq S < \frac{m_1}{s\sigma_1} + \frac{m_2}{s\sigma_2}, \\
    \vdots \\
    Z_d & \text{If } \sum_{i=1}^{d-1}\frac{m_i}{s\sigma_i} \leq S < 1. 
    \end{cases}
\end{equation}
To compute $\Var\Exp[\phi(X)\mid Z]$, define $K\defeq \Exp[\phi(X)\mid Z]$, then by the law of total variance, we have:
\begin{equation*}
    \Var\Exp[\phi(X)\mid Z] = \Var(K) = \Exp[\Var(K\mid S)] + \Var\Exp[K\mid S].
\end{equation*}
We first compute $\Var\Exp[K\mid S]$:
\begin{align*}
    \Var\Exp[K\mid S] &= \Var\Exp[\Exp[\phi(X)\mid Z]\mid S] \\
                      &= \Var\Exp[\phi(X)\mid S] && \text{(The law of total expectation)} \\
                      &= \Var\Exp[\phi(X)] && \text{($\phi(X)\perp S$)} \\
                      &= 0.
\end{align*}
On the other hand, for $\Exp[\Var(K\mid S)]$, we have:
\begin{align*}
    \Exp[\Var(K\mid S)] &= \sum_{i = 1}^d \Pr(S = i)\cdot\Var(K\mid S = i) \\
                        &= \sum_{i = 1}^d \frac{m_i}{s\sigma_i}\cdot\Var(K\mid S = i) \\
                        &= \sum_{i = 1}^d \frac{m_i}{s\sigma_i}\cdot\Var\Exp[\phi(X)\mid Z, S=i] \\
                        &= \sum_{i = 1}^d \frac{m_i}{s\sigma_i}\cdot\Var\Exp[\phi(X)\mid Z_i] \\
                        &= \sum_{i = 1}^d \frac{m_i}{s\sigma_i}\cdot sM_i \\
                        &= M,
\end{align*}        
completing the proof.
\end{proof}
To prove Theorem~\ref{thm:achievable}, in what follows we prove that under the regularity condition, all the bounds in Theorem~\ref{thm:maxvar}, Theorem~\ref{thm:minvar} and Theorem~\ref{thm:lower_bound} are tight. 

\subsubsection{Tightness of the Bound in Theorem~\ref{thm:maxvar}}
In the proof of Theorem \ref{thm:maxvar}, we showed an upper bound on the optimal solution via an SDP relaxation. Therefore, the upper bound is achievable whenever the SDP relaxation is tight. 
\begin{corollary}
When $(X, \phi)$ is regular, the upper bound in Theorem \ref{thm:maxvar} is achievable. 
\end{corollary}
\begin{proof}
This immediately follows from the proof of Theorem~\ref{thm:maxvar} and Theorem~\ref{thm:regular}: if $(X, \phi)$ is regular, then there exists $Z = g(X)$ such that $\Var\E[\phi(X)\mid Z] = V^*$, where $V^*$ is the optimal solution of the SDP relaxation in the proof of Theorem~\ref{thm:maxvar}.
\end{proof}

\subsubsection{Tightness of the Bound in Theorem~\ref{thm:minvar}}
In the proof of Theorem \ref{thm:minvar}, we showed a lower bound on the optimal solution via an SDP relaxation. Therefore, the lower bound is achievable whenever the SDP relaxation is tight. 
\begin{corollary}
When $(X, \phi)$ is regular, the lower bound in Theorem \ref{thm:minvar} is achievable. 
\end{corollary}
\begin{proof}
This immediately follows from the proof of Theorem~\ref{thm:minvar} and Theorem~\ref{thm:regular}: if $(X, \phi)$ is regular, then there exists $Z = g(X)$ such that $\Var\E[\phi(X)\mid Z] = V^*$, where $V^*$ is the optimal solution of the SDP relaxation in the proof of Theorem~\ref{thm:minvar}.
\end{proof}

\subsubsection{Tightness of the Bound in Theorem~\ref{thm:lower_bound}}
In the proof of Theorem \ref{thm:lower_bound}, we showed a lower bound on the tradeoff via an SDP relaxation. Therefore, the lower bound is achievable whenever the SDP relaxation is tight. 

\begin{proof}
From the proof of Theorem \ref{thm:lower_bound}, we can see that if there exists $Z = g(X)$, such that
\begin{align*}
    \Var\E[ \phi(X)\mid Z] =  \Sigma^{1/2}u_d u_d^T\Sigma^{1/2},
\end{align*}
where $u_d$ is the unit eigenvector of $R$ with eigenvalue $\sigma_d(R)$, then the equality is achievable. It is easy to see that 
\begin{align*}
    \Sigma \succeq \Sigma^{1/2}u_d u_d^T\Sigma^{1/2} \succeq 0.
\end{align*}
Therefore, choosing $M =\Sigma^{1/2}u_d u_d^T\Sigma^{1/2}$ in the definition of regularity guarantees the existence of $Z$. Hence we have completed the proof.
\end{proof}

\subsection{Proof without RKHS assumption}
The following lemma is crucial.
\begin{lemma}
\label{lem:generalized_lower_bound} 
For any $f, g $,
\begin{align} \label{eqn:first_ineq}
\begin{aligned}
 &\Var\E  [f(X)| Z]  - \Var\E  [ g(X) |Z ]  \\
\ge& \frac12 \Big\{
\Var[f(X)] - \Var[g(X)]
- \sqrt{
	(\Var[f(X)]+ \Var[g(X)]
)^2
	- 4\Cov(f(X), g(X))^2
}
\Big\}
\end{aligned}
\end{align}
and
\begin{align} \label{eqn:second_ineq}
\begin{aligned}
&\Var\E  [f(X)| Z]  - \Var\E  [ g(X) |Z ]  \\
\le& \frac12 \Big\{
\Var[f(X)] - \Var[g(X)]
+ \sqrt{
	(\Var[f(X)]+ \Var[g(X)]
	)^2
	- 4\Cov(f(X), g(X))^2
}
\Big\}
\end{aligned}
\end{align}
\end{lemma}
\begin{proof}
After re-organizing the terms and multiplying a factor of $2$ on both sides, the inequality \eqref{eqn:first_ineq} is equivalent to the following:
\begin{align*}
 &\sqrt{
	(\Var[f(X)]+ \Var[g(X)]
)^2
	- 4(\Cov(f(X), g(X))^2
}\\
\ge & \Var[f(X)]- \Var[g(X)] - 2 \Var\E  [f(X)| Z] + 2 \Var\E  [g(X)| Z] \\
= & 
-\Var[f(X)]+ \Var[g(X)]
+2\E \Var \left(f(X)| Z\right)  -2 \E \Var \left( g(X) |Z \right)
\end{align*}
In the last step, we used the law of total variance. Similarly, the inequality \eqref{eqn:second_ineq} is equivalent to:
\begin{align*}
&\sqrt{
	(\Var[f(X)]+ \Var[g(X)]
	)^2
	- 4(\Cov(f(X), g(X))^2
}\\
\ge  & 
\Var[f(X)]- \Var[g(X)]
-2\E \Var \left(f(X)| Z\right)  +2 \E \Var \left( g(X) |Z \right)
\end{align*}
Therefore, it suffices to prove that 
\begin{align}
&\sqrt{
	(\Var[f(X)]+ \Var[g(X)]
	)^2
	- 4(\Cov(f(X), g(X))^2
} \nonumber\\
\ge &
\Big|
\Var[f(X)]- \Var[g(X)]
- 2\E \Var \left(f(X)| Z\right)  +2 \E \Var \left( g(X) |Z \right)
\Big| \label{eqn:crucial}
\end{align}
Let $p(X) = f(X) +g(X)$ and $q(X) = f(X) -g(X)$. Notice that
\begin{align*}
\Var[f(X)]- \Var[g(X)] &= \Cov(p(X) , q(X))\\
\Var(f(X)|Z) - \Var(g(X)|Z)  &= \Cov(p(X) , q(X) |Z)
\end{align*}
and
\begin{align*}
(\Var[f(X)] + \Var[g(X)] 
)^2 - 4\Cov(f(X), g(X))^2
 = \Var[p(X)] \Var[q(X)].
\end{align*}
The inequality \eqref{eqn:crucial} can be further simplified to 
\begin{align*}
\sqrt{
	\Var[p(X)] \Var[q(X)]
}
\ge 
\Big| \Cov(p(X) , q(X))
- 	2 \E \Cov(p(X) , q(X)|Z) \Big|.
\end{align*}
When $p(X)=c$ or $q(X)=c$ for some constant $c$ a.s., both sides equal to zero, hence the inequality obviously holds. 

For the other cases, we use the following simple observation: $B \leq \sqrt{AC}, A > 0, C > 0 \Leftrightarrow \inf_{t \in \R} At^2+2Bt+C \geq 0 $.  Therefore, we only need to prove that for any $t \in \R$, 
\begin{align*}
& t^2 \Var[p(X)]
+ 2t\left(\Cov(p(X) , q(X))
- 	2 \E \Cov(p(X) , q(X)|Z) \right) 
+  \Var[q(X)] \geq 0
\end{align*}
In fact, the inequality we wanted to prove can be re-written as:
\begin{equation}
\Var[t\cdot p(X) + q(X)] - 4 \E \Cov(t\cdot p(X), q(X)|Z) \geq 0.
\end{equation}

Notice that 
\begin{equation}
    4 \Cov(t\cdot p(X), q(X)|Z) = \Var[t\cdot p(X)+ q(X) |Z] - \Var[t\cdot p(X)- q(X) |Z]
\end{equation}
We have
\begin{align*}
& \Var[t\cdot p(X) + q(X)] - 4 \E \Cov(t\cdot p(X), q(X)|Z) \\
=& \left(\Var[t\cdot p(X) + q(X)]  -\E \Var[t\cdot p(X) + q(X)| Z] \right) + \E \Var[t\cdot p(X)- q(X) |Z]\\
=& \Var(\E[t\cdot p(X) + q(X)|Z]) +\E \Var[t\cdot p(X)- q(X) |Z] \\
\ge& 0,
\end{align*}
where in the second to last step we used the law of total variance. Therefore we have completed the proof.
\end{proof}
Here are two corollaries of the lemma that will be useful for proving Theorem \ref{thm:maxvar_noisy} and \ref{thm:minvar_noisy}.
\begin{corollary}\label{cor:first}
    Suppose the random variable $Z$ satisfies $\Var(\E[g(x) |Z]) = 0$, then we have the following upper bound:
    \begin{equation}
    	 \Var \E [f(X)| Z] \le \Var[f(X)]- \frac{\Cov(f(X), g(X))^2}{\Var[g(X)]}.
    \end{equation}
\end{corollary}
\begin{proof}
	By applying the second inequality in Lemma \ref{lem:generalized_lower_bound} to functions $f$ and $\sqrt{\lambda} \cdot g$, for any $\lambda \ge 0$,
	\begin{align}\label{eqn:first_step_proof_cor2}
	\begin{aligned}
	&\Var \E [f(X)| Z] - \lambda\Var \E [g(X)| Z]\\
	\le& \frac12 \Big\{
	\Var[f(X)] - \lambda \Var[g(X)]
	+ \sqrt{
		(\Var[f(X)]+ \lambda \Var[g(X)]
		)^2
		- 4\lambda \Cov(f(X), g(X))^2
	}
	\Big\}
	\end{aligned}
	\end{align}
	Since $\Var \E [g(X)| Z] = 0$,  Equation \eqref{eqn:first_step_proof_cor2} is equivalent to:
	\begin{align}
	\begin{aligned}
	&\Var \E [f(X)| Z]   \\
	\le& \frac12 \Big\{
	\Var[f(X)] - \lambda \Var[g(X)]
	+ \sqrt{
		(\Var[f(X)]+ \lambda \Var[g(X)]
		)^2
		- 4\lambda \Cov(f(X), g(X))^2
	}
	\Big\}\\
	&= \frac{2\lambda \Var[f(X)]\Var[g(X)]-2\lambda \Cov(f(X), g(X))^2}{-\Var[f(X)] + \lambda \Var[g(X)] +\sqrt{
			(\Var[f(X)]+ \lambda \Var[g(X)]
			)^2
			- 4\lambda \Cov(f(X), g(X))^2
	} },
	\end{aligned}
	\end{align}
	Taking the limit $\lambda \rightarrow \infty$ completes the proof.
\end{proof}

\begin{corollary}\label{cor:second}
	Suppose the random variable $Z$ satisfies $\Var(\E[g(x) |Z]) =\Var(g(x))$, then we have the following lower bound:
	\begin{equation}
	\Var \E [f(X)| Z] \ge \frac{\Cov(f(X), g(X))^2}{\Var[g(X)]}	
	\end{equation}
\end{corollary}
\begin{proof}
	By applying the first inequality in Lemma \ref{lem:generalized_lower_bound} to functions $f$ and $\sqrt{\lambda} \cdot g$, for any $\lambda \ge 0$,
	\begin{align}\label{eqn:first_step_proof_cor1}
	\begin{aligned}
	&\Var \E [f(X)| Z] - \lambda\Var \E [g(X)| Z]  \\
	\ge& \frac12 \Big\{
	\Var[f(X)] - \lambda \Var[g(X)]
	- \sqrt{
		(\Var[f(X)]+ \lambda \Var[g(X)]
		)^2
		- 4\lambda \Cov(f(X), g(X))^2
	}
	\Big\}
	\end{aligned}
	\end{align}
	Since $\Var(\E[g(X) |Z]) =\Var[g(X)]$, Equation \eqref{eqn:first_step_proof_cor1} is equivalent to:
	\begin{align}
	\begin{aligned}
	&\E \Var \left(f(X)| Z\right)   \\
	\ge& \frac12 \Big\{
	\Var[f(X)] + \lambda \Var[g(X)]
	- \sqrt{
		(\Var[f(X)]+ \lambda \Var[g(X)]
		)^2
		- 4\lambda \Cov(f(X), g(X))^2
	}
	\Big\}\\
	&= \frac{2\lambda \Cov(f(X), g(X))^2}{\Var[f(X)] + \lambda \Var[g(X)] +\sqrt{
			(\Var[f(X)]+ \lambda \Var[g(X)]
			)^2
			- 4\lambda \Cov(f(X), g(X))^2
	} },
	\end{aligned}
	\end{align}
	Taking the limit $\lambda \rightarrow \infty$ completes the proof.
\end{proof}

Now we are ready to prove Theorem \ref{thm:maxvar_noisy}, \ref{thm:minvar_noisy} and \ref{thm:lower_bound_noisy}. 

\maxvarnoisy*
\begin{proof}[Proof of Theorem~\ref{thm:maxvar_noisy}]
We apply corollary \ref{cor:first} by choosing $f:= f_Y^*$ and $g:= f_A^*$, this gives us the following lower bound when $\Var(\E[f_A^*(X)|Z] ) = 0$
\begin{align*}
	 \Var (\E [f_Y^*(X)| Z]) 
	 & \le \Var [f_Y^*(X)] - \frac{\cov(f_Y^*(X), f_A^*(X))^2}{\Var [f_A^*(X)]} \\
	 &= \Var [f_Y^*(X)] (1- \rho_{YA}^2).
\end{align*}
Therefore we have completed the proof.
\end{proof}
\minvarnoisy*
\begin{proof}[Proof of Theorem~\ref{thm:minvar_noisy}]
	We apply corollary \ref{cor:first} by choosing $f:= f_A^*$ and $g:= f_Y^*$, this gives us the following lower bound when $\Var(\E[f_Y^*(X)|Z] ) = \Var[f_Y^*(X)]$
	\begin{align*}
	\Var (\E [f_A^*(X)| Z]) &\ge \frac{\Cov(f(X), g(X))^2}{\Var[g(X)]}	 \\
	&= \Var [f_Y^*(X)] \rho_{YA}^2.
	\end{align*}
	Therefore we have completed the proof.
\end{proof}
\lagrangian*
\begin{proof}[Proof of Theorem~\ref{thm:lower_bound_noisy}]

The desired inequality follows directly from applying Lemma \ref{lem:generalized_lower_bound} above to $f:= f_Y^*$ and $g:= \sqrt{\lambda}f_A^*$.
\end{proof}







\end{document}